\documentclass{article}
\usepackage{algorithm}

\PassOptionsToPackage{square,numbers}{natbib}
\usepackage[preprint]{neurips_2026}

\usepackage[utf8]{inputenc} 
\usepackage[T1]{fontenc}    
\usepackage{hyperref}       
\usepackage{url}            
\usepackage{booktabs}       
\usepackage{amsfonts}       
\usepackage{nicefrac}       
\usepackage{microtype}      
\usepackage{xcolor}         

\usepackage{wrapfig}
\usepackage{microtype}
\usepackage{graphicx}
\usepackage{subcaption}
\usepackage{booktabs} 
\usepackage{algorithmic}
\usepackage{bbm}

\usepackage{amsmath}
\usepackage{amssymb}
\usepackage{mathtools}
\usepackage{amsthm}

\usepackage[capitalize,noabbrev]{cleveref}

\theoremstyle{plain}
\newtheorem{theorem}{Theorem}[section]

\newtheorem{corollary}[theorem]{Corollary}
\theoremstyle{definition}

\newtheorem{assumption}[theorem]{Assumption}
\theoremstyle{remark}

\newcommand{\argmax}{\arg\max}

\usepackage[textsize=tiny]{todonotes}

\usepackage[inline]{enumitem}

\title{Provably Reduced Sample Cost in Prior-Guided Hyperparameter Optimization}

\author{%
  Leona Hennig \\
  Institute of Artificial Intelligence \\
  Leibniz University Hanover, Germany \\
  \AND
  Jasmin Brandt \\
  University of Bielefeld \\
  Bielefeld, Germany \\
  \AND
  Lukas Fehring \\
  Institute of Artificial Intelligence \\
  Leibniz University Hanover, Germany \\
  \AND
  Barbara Hammer \\
  University of Bielefeld \\
  Bielefeld, Germany \\
  \AND
  Marius Lindauer \\
  Institute of Artificial Intelligence, Leibniz University Hanover \& L3S Research Center \\
  Hanover, Germany \\
  \AND
  Marcel Wever \\
  Institute of Artificial Intelligence, Leibniz University Hanover \& L3S Research Center \\
  Hanover, Germany \\
}

\begin{document}

\maketitle

\begin{abstract}
Large-scale hyperparameter optimization (HPO) in automated machine learning (AutoML) consumes substantial computational resources, raising concerns about scalability and energy efficiency. While existing methods use prior information heuristically to accelerate black-box and multi-fidelity settings, they lack a quantitative characterization of how prior informativeness reduces sample complexity.
We provide the first distribution-dependent sample complexity bounds for multi-fidelity HPO with priors via fixed-budget best-arm identification. By modeling priors directly over arm means as configuration performance, we derive explicit, distribution-dependent error bounds that quantify the relationship between prior quality and evaluation budget. Our analysis shows that informative priors yield provable reductions in the number of required evaluations; crucially, baseline performance is recovered, without degradation, with uninformative or misleading priors, establishing a robustness guarantee absent from prior heuristic approaches.
We conduct proof-of-concept experiments on a synthetic benchmark and on LCBench to confirm our theoretical results, achieving up to 80\% budget reduction while retaining solution quality. Together, our results provide a principled theoretical foundation for understanding how prior knowledge governs computational efficiency for green AutoML.
\end{abstract}

\section{Introduction}\label{sec:introduction}
While automated machine learning (AutoML), specifically hyperparameter optimization (HPO), is essential to modern machine learning pipelines \citep{hutter-book19a, eggensperger-neuripsdbt21a, muller-icml23a}, its computational cost is substantial. State-of-the-art HPO methods often require hundreds of model evaluations, driving up the costs as model sizes and dataset volumes scale. Likewise, the resulting energy consumption increasingly challenges the scalability and sustainability of machine learning systems. Consequently, a primary goal of Green AutoML \citep{schwartz-arxiv19a,tornede-jair23a} is to reduce this evaluation burden without sacrificing performance. 

A way to improve the efficiency of HPO is to leverage prior knowledge about promising configurations. Such priors may arise from meta-learning across tasks~\citep{feurer-aaai15a, rijn-kdd18a, perrone-neurips19a,li2026lamda, rothfuss2021pacoh}, expert knowledge~\citep{hvarfner-iclr22a,mallik-metaws22a, seng-automlconf25a}, or structural properties of learning algorithms~\citep{klein-iclr17a}. In fact, many practical AutoML systems already incorporate prior knowledge, either explicitly \citep{feurer-aaai15a,feurer-jmlr22a} or implicitly \citep{erickson-arxiv20a,mohr-ml23a}. 

Despite widespread use, existing approaches leverage priors primarily as heuristic accelerators. In particular, no previous work provides a principled characterization of how the informativeness of a prior translates into reductions in evaluation budgets for HPO for multi-fidelity settings. One example is PriorBand \citep{mallik-neurips23a}, which accelerates HPO through 
prior-guided sampling.  It remains unclear when and to what extent prior knowledge can provably improve the sample-complexity bounds of (near-)optimal configuration identification. 

Equally absent is any robustness guarantee.
It is unknown whether incorporating a misleading prior can degrade identification performance relative to the prior-free baseline. Our work fills this gap by providing the first distribution-dependent sample complexity bounds for prior-guided multi-fidelity HPO, yielding 
actionable guarantees on the minimum prior quality required for budget 
savings and formal robustness certificates against misleading priors.

We address this theoretical gap from the perspective of fixed-budget Best-Arm Identification (BAI) in stochastic multi-armed bandits (MABs)~\citep{audibert2010best}. We model a BAI problem in which configurations correspond to arms with unknown (full-fidelity) mean performance, observed through noisy, low-fidelity evaluations. Unlike prior-free analyses, we explicitly incorporate prior distributions over arm means and study their effect on identification accuracy.

Our key finding is that prior quality provably governs sample complexity. Informative priors cause exponential suppression of suboptimal arms, accelerating the decline in misidentification probability and significantly reducing the required evaluation budget. Conversely, we prove that uninformative or misleading priors cannot degrade performance below the prior-free baseline, a formal robustness guarantee that distinguishes our approach from heuristic prior integration. This establishes a rigorous characterization of the relationship between prior informativeness and computational efficiency.


Building on the Bayesian fixed-budget analysis of \citet{atsidakou-arxiv22a}, we derive distribution-dependent error bounds for identifying the $\epsilon$-best arm on a fixed budget. This objective is particularly well suited to HPO, where many configurations often yield near-identical performance, making it unnecessary and costly to identify the exact optimum.
Our analysis reveals that the misidentification probability can be decomposed into a \emph{prior effect}, determined by the gap between ground truth and prior, and a \emph{sampling effect}, determined by the evidence from evaluations. This decomposition provides explicit guarantees that quantify how informative priors reduce sample complexity, and formally certifies that performance cannot fall below the prior-free baseline regardless of prior quality.

We extend our framework to multi-fidelity optimization, which is essential for HPO with expensive deep learning models \citep{jamieson-aistats16a, li-iclr17a, falkner-ai18a}. Partial training runs provide low-cost but noisy signals about full-budget performance. We employ Gaussian processes (GPs) as probabilistic estimators to extrapolate learning curves across fidelities and integrate them into a Successive Halving (SH) framework guided by priors \citep{jamieson-aistats16a}, enabling theoretically grounded early stopping, while keeping control over uncertainty. The uncertainty introduced by GP-based 
extrapolation across fidelities is formally incorporated into our 
theoretical framework in Theorem~\ref{theorem worst case}.

Empirical results on a synthetic and a common real-world HPO benchmark confirm our theoretical findings: informative priors lead to significant reductions in the evaluation budget through early termination of HPO in line with our theoretical predictions, and energy consumption, while misleading priors do not impair performance as guaranteed by our bounds. We show compatibility with Hyperband\cite{li-iclr17a} and PriorBand~\cite{mallik-neurips23a}. Together, our results provide the first principled characterization of how prior knowledge influences evaluation efficiency in multi-fidelity HPO, establishing a theoretical foundation for prior-guided and green AutoML.

\paragraph{Contributions}
Our contributions are threefold:
\begin{enumerate}
    \item We extend fixed-budget Best-Arm Identification analysis with priors on arm means and derive explicit bounds on sample complexity that characterize the relationship between prior quality and evaluation budget, with a formal robustness guarantee against misleading priors.
    \item We propose Prior-Guided Successive Halving (PSH; Algorithm \ref{alg:mfpsh_short}), which applies our theoretical framework to HPO by using prior information to estimate full-budget reward from partial observations, significantly reducing costs in resource-constrained settings.
    \item We showcase empirically that PSH achieves comparable or superior performance with significantly reduced computational budgets, contributing to greener AutoML.
\end{enumerate}

\section{Related Work}\label{sec:related-work}

This paper builds on and connects three research threads: \begin{enumerate*}[label=(\roman*)]
    \item bandit-based Best-Arm Identification for HPO,
    \item methods that exploit priors on good configurations or optima, and
    \item work on green AutoML and energy-aware evaluation.
\end{enumerate*}

\subsection{Bandit-Based Best-Arm Identification and HPO}
A major research line views HPO as adaptive resource allocation. Hyperband \citep{li-iclr17a} popularized this by combining re-iterated random sampling with SH \citep{jamieson-aistats16a} to efficiently allocate budgets and leverage early-stopping. Related MAB formulations also appear in the more general setting of algorithm configuration \citep{schede-jair22a}, e.g., AC-Band \citep{brandt-aaai23a}, which adopts a combinatorial bandit approach with theoretical guarantees and strong empirical performance in online settings.
The theoretical approach closest to ours is Bayesian fixed-budget BAI \citep{atsidakou-arxiv22a}, proposing a Bayesian elimination strategy and deriving distribution-dependent bounds for misidentification. We quantify how prior quality affects evaluation budgets and establish robustness against non-informative priors and the uncertainty induced by the GP in an $\epsilon$-best arm scenario. \citet{brandt-ijcai24a} examine SH with retroactively increased budgets, offering theoretical guarantees and improved computational efficiency, whereas PaSHA \citep{bohdal-iclr23a} increases the highest evalauted fidelity until configuration rankings stabilize.

\subsection{Priors on Optima for Faster Optimization}

A second line of research accelerates HPO by leveraging prior knowledge, such as expert priors and meta-learning signals. \citet{perrone-neurips19a} use historical data to shrink or shape search spaces, though without theoretical guarantees. Other methods incorporate explicit priors about promising configurations: \citet{souza-ecml21a, hvarfner-iclr22a} augment acquisition functions with user-defined priors and offer regret guarantees and empirical accelerations, while \citet{mallik-neurips23a} apply this idea to Hyperband \citep{li-iclr17a} without providing theoretical analysis. More recently, \citet{seng-automlconf25a} proposed interactive optimization via runtime priors, though without support for multi-fidelity evaluation. Complementary to these approaches, \citet{li2026lamda} learn data-driven priors from low-fidelity training signals to accelerate high-fidelity evaluations with a focus on improving anytime performance rather than fixed-budget guarantees. In contrast, we characterize how prior informativeness affects the budget for $\epsilon$-Best Arm Identification.
While PriorBand \citep{mallik-neurips23a} extends Hyperband using a dynamic strategy interpolating between prior-guided and random sampling, it lacks theoretical guarantees on how prior quality impacts evaluation costs. Our approach is therefore complementary; we evaluate PriorBand as an empirical baseline in Section~\ref{sec:emipirical_results}.

\subsection{Green AutoML and Energy-Aware Optimization}
Green AutoML has emerged as a response to the growing resource footprint of AutoML systems. \citet{tornede-jair23a} articulate the Green AutoML perspective, discussing how to quantify AutoML footprints and outlining strategies and checklists for sustainable AutoML systems and research. Complementary work studies energy metrics and energy-aware evaluation practices in AutoML and HPO. For example, \citep{castellanosnieves-applsci23a,castellanosnieves-applsci24a, giovanelli-aaai24a} investigate energy-efficiency metrics and strategies to make AutoML/HPO more sustainable. Related Green AI efforts quantify carbon impacts in ML workflows such as fine-tuning \citep{ordoumpozanis-3ict24a}.
Our work adds to this literature by providing a principled theoretical account of how prior knowledge can robustly reduce evaluation budgets.

\section{Background}
This section reviews the foundations of hyperparameter optimization (HPO), formalizes the Best-Arm Identification (BAI) setting, and motivates the use of Bayesian priors in multi-fidelity settings.

\paragraph{Hyperparameter Optimization}
The increasing complexity of modern machine learning (ML) approaches has amplified the need for HPO to reliably obtain high-performing models \citep{bischl-dmkd23a}.
Let $\Lambda\subseteq\mathbb{R}^d$ be a $d$-dimensional hyperparameter space. 
Let $\mathcal{D} = \{(x_i, y_i)\}_i \subset \mathcal{X} \times \mathcal{Y}$ be a dataset, which is split into
disjoint training, validation, and test sets $\mathcal{D}_{\mathrm{train}}, \mathcal{D}_{\mathrm{val}}, \mathcal{D}_{\mathrm{test}}$.
A learning algorithm $A : \mathcal{D}_{\mathrm{train}} \times \Lambda \to \mathcal{M}$ maps a configuration $\lambda \in \Lambda$ to a trained model $M_\lambda = A(\mathcal{D}_{\mathrm{train}}, \lambda)$ from a model class $\mathcal{M}$.
The performance of $M_\lambda$ is assessed via a loss function $\mathcal{L}(M_\lambda,\, \mathcal{D}_{\mathrm{val}})$.
The HPO objective is to identify a configuration minimizing the validation loss:
$
    \lambda^* \in \operatorname*{argmin}_{\lambda \in \Lambda}
  \mathcal{L}\!\bigl(A(\mathcal{D}_{\mathrm{train}}, \lambda),\, \mathcal{D}_{\mathrm{val}}\bigr).
$
Each configuration corresponds to a candidate ML model whose performance is revealed only through noisy and computationally expensive evaluations. Let $\mathcal{K}$ denote a finite set of configurations; evaluating a configuration produces a stochastic observation of its performance, typically obtained by training a model and validating it on $\mathcal{D}_{\mathrm{val}}$.

\paragraph{Multi-Fidelity Optimization}
In practice, applications often impose evaluation budgets.
Thus, candidate configurations are often evaluated in a multi-fidelity (MF) 
scheme.
Let $\mathcal{F}$ denote a set of fidelity levels, meaning different resource levels, such as training epochs or data subsampling.
$b \in [0, B]$ denotes the fidelity.

For any configuration $\lambda\in\Lambda$, lower fidelity evaluations $f_b(\lambda)$ provide cheaper but noisier estimates of the full budget objective $F(\lambda) = f_B(\lambda)$.
Algorithms like Successive Halving (SH)~\citep{jamieson-aistats16a} and Hyperband~\citep{li-iclr17a} exploit this structure by evaluating many configurations at low fidelity and allocating higher budgets to the most promising ones.

\paragraph{Best-Arm Identification}
This objective is naturally captured by the Best-Arm Identification (BAI) framework in stochastic bandits, where configurations are modeled as arms with unknown mean rewards. Unlike regret minimization, BAI focuses on identifying the single best (or $\epsilon$-best) arm under a fixed evaluation budget, directly reflecting the deployment-oriented nature of HPO.
Bandit-based HPO methods, most notably based on Successive Halving, such as Hyperband \citep{li-iclr17a}, exploit multi-fidelity evaluations to eliminate suboptimal configurations early when transitioning between fidelities and concentrate resources on promising candidates. However, these methods typically do not account for prior knowledge about configuration quality.

\paragraph{Including Prior Knowledge}
In many cases, prior information about promising configurations is available from previous tasks or from expert knowledge, which can be used to improve anytime performance \citep{mallik-neurips23a}. This motivates prior-guided BAI, where priors over arm mean rewards are incorporated into the sampling and elimination process. Subsequently, we formalize this setting and analyze how prior quality affects the evaluation budget required to identify an $\epsilon$-best configuration.

\section{Problem Setting}\label{sec:problem-setting}
We formalize the hyperparameter optimization (HPO) task as a Bayesian multi-fidelity (MF) Best-Arm Identification (BAI) problem, where configurations are treated as stochastic arms with latent learning curve properties.

\textbf{Bandit Model}
Let $\mathcal{K} = \{1, \dots, K\}$ be a set of $K$ arms. Associated with each arm $i\in\mathcal{K}$ is an unknown learning curve function $L_i: [0, B] \to \mathbb{R}$, where $b \in [0, B]$ denotes the fidelity and $L_i(B)$ represents the final performance on the full budget.
We define the true mean of arm $i$ as its value at maximum fidelity, \mbox{$\mu_i \coloneqq L_i(B)$}.
At each time step $t=1, \dots, N$, the learner selects an arm $i_t$ and a fidelity $b_t$ to observe a noisy sample \mbox{$y_{t} = L_{i_t}(b_t) + \eta_t$}, where \mbox{$\eta_t \sim \mathcal{N}(0, \sigma_{i_t}^2)$} is independent Gaussian noise. We denote by $\Sigma$ the GP posterior predictive variance at target fidelity $B$, 
updated dynamically each round.
The goal is to recommend the optimal arm $j^*$ that maximizes the true mean: 
\begin{equation}
    j^* \in \operatorname{argmax}_{k \in \mathcal{K}} \mu_k.
\end{equation}

\textbf{Bayesian Priors and Multi-Fidelity Update}
For each arm in $\mathcal{K}$, we assume that the user injects prior knowledge via a Gaussian Process (GP) belief $\mathcal{H}_0$ over the predicted performance at target fidelity $B$.
At initialization, this induces a marginal belief 
$    \mu_i \sim \mathcal{N}(\nu_{i}, \sigma_{0}^2).$
with $\nu_i\in[0,1]$ and $\sigma_0> 0$. A smaller $\sigma_0$ indicates high expert certainty, whereas a larger value places more emphasis on the GP-based observations.

\textbf{Objective ($\epsilon$-Best Arm Identification)}
Standard fixed-budget BAI formulations typically seek to minimize the misidentification probability $\mathbb{P}(J \neq j^*)$ \citep{audibert2010best} for an estimated best arm $j$ by the considered BAI method. However, this worst-case perspective is dominated by instances with infinitesimally small gaps \citep{even2002pac}. Searching for the single best configuration in an infinite configuration space equals searching for a needle in a haystack. 
We argue that distinguishing between configurations with performance differences smaller than a tolerance $\epsilon$ yields negligible practical utility. Given the inherent noise in performance estimates, attempting to resolve differences below this threshold risks fitting the noise rather than the signal, instead of improving generalization.
Therefore, we adopt the $\epsilon$-BAI objective. We treat the selection of any arm $j$ as a success provided that $\mu_{j^*} - \mu_j \le \epsilon$.

To derive our stopping condition, we first analyze the expected Bayesian risk of the algorithm under a fixed budget allocation $N$:
\begin{equation}
\ \mathbb{E}_{{\mu} \sim \mathcal{H}_0}\left[ \mathbb{I}(\mu_{j^*} - \mu_J > \epsilon) \right].
\end{equation}

\paragraph{Remark (Fixed-Budget vs.\ Fixed-Confidence)} Our framework is 
set in the fixed-budget BAI setting: PSH operates under a hard resource cap $N$. The early stopping rule introduces a fixed-confidence aspect in the sense that termination is triggered once the combined evidence of prior and observations meets a confidence threshold for $\epsilon$-BAI. Crucially, however, this threshold is derived from the fixed-budget bound itself (Theorem~\ref{theorem worst case}), so PSH defaults to standard SH if the threshold is not met, guaranteeing that 
resource consumption is upper-bounded by $N$ regardless of prior quality.

\section{Prior-Guided Successive Halving}\label{sec:prior guided sh}
We propose Prior-Guided Successive Halving (PSH), a Bayesian adaptation of the standard Successive Halving (SH) strategy \citep{jamieson-aistats16a} detailed in Algorithm \ref{alg:mf_prior_sh}. SH is constrained by a fixed, predetermined budget $N$. We treat this budget as a maximum resource cap, while leveraging prior information to enable early stopping at each fidelity. The priors are generated using different sampling mechanisms that encode different assumptions about the configuration performance. For details on the priors considered here, see Appendix \ref{sec:appendix prior generation}.

\begin{wrapfigure}{r}{0.62\textwidth}
  \vspace{-1em}
  \noindent\rule{\linewidth}{0.4pt} \\
  \refstepcounter{algorithm}
  \textbf{Algorithm 1} PSH (Simplified)
  \label{alg:mfpsh_short} \\
  \rule{\linewidth}{0.4pt}
  \begin{algorithmic}[1]
    \REQUIRE Arms $\mathcal{K}$, priors $(\nu_j,\sigma_0^2)_{j\in\mathcal{K}}$
    \STATE $S_0 \leftarrow \mathcal{K}$
    \FOR{round $r = 0,1,\dots,R-1 \:\text{with}\: R = \lceil \log_\eta(K) \rceil$}
        \STATE Evaluate each arm $j\in S_r$ for budget $n_r$;
        \\ update GP predictions $\hat{\mu}_{j,r}\:\forall j\in S_r$,
        \\ calculate $N_\text{stop}$
        \STATE Update dataset $\mathcal{D}_j \leftarrow \{(t, y_{j,t})\}_{t=1}^{N_{\text{used}_j}}$
        \STATE $j^* \leftarrow \argmax_{j \in S_r} \hat{\mu}_{j,r}$
        \IF{$N_{\mathrm{used}} \ge N_{\mathrm{stop}}$}
          \STATE \textbf{return} $j^*$
        \ENDIF
        \STATE $S_{r+1} \leftarrow$ top $\lceil |S_r|/\eta \rceil$ arms by $\hat{\mu}_{j,r}$
    \ENDFOR
    \STATE \textbf{return} $j^*$
  \end{algorithmic}
  \rule{\linewidth}{0.4pt}
  \vspace{-1em}
\end{wrapfigure}
%

The algorithm proceeds in \mbox{$R = \lceil \log_\eta K \rceil$} rounds, where $\eta \geq 2$ is the elimination rate controlling the fraction of 
arms retained per round; we set $\eta = 2$ throughout.
Let $S_r$ be the set of active arms in round $r$, with $S_0 = \{1, \dots, K\}$ and $|S_r| = \lceil K/2^r \rceil$.
In each round~$r$, the algorithm allocates an equal budget $n_r$ to all arms in $S_r$. The total budget used so far is $N_\text{used}$. The posterior belief is updated using the GP inference. The algorithm retains the top $K_{r+1} = \lceil |S_r|/2 \rceil$ arms with the highest posterior means. A simplified pseudocode of PSH can be found in Algorithm \ref{alg:mfpsh_short}. We additionally provide a detailed algorithm in Appendix \ref{sec:appendix algorithm}.

Our theoretical analysis in Section \ref{sec:theoretical-guarantees} (Theorem \ref{theorem worst case}) defines a sufficient sample complexity $N_\text{stop}$ required to guarantee $\epsilon$-accuracy given the prior $\mathcal{H}_0$. We implement this bound as a dynamic stopping condition. In every round, the algorithm computes the budget $N_{stop}$ required to distinguish the current best candidate $j^*$ from the remaining arms with probability $1-\delta$. Theorem \ref{theorem worst case} uses the current estimated gaps $\Delta_{\epsilon,j}$ and the original prior bias of the means and variances $(\nu_{j^*} - \nu_j)/\sigma^2_0$.
To account for the multi-fidelity process of the evaluations, we replace the static sample mean estimate with a trajectory-based prediction. We model the performance of each arm $j$ as a function of fidelity $t$, $f_j(t) \sim \mathcal{GP}(m_j(t), k(t, t'))$, where the mean function is initialized to the prior belief, $m_j(t) = \nu_j$. At each round $r$, conditioned on the observed partial learning curve $\mathcal{D}_j = \{(t, y_{j,t})\}_{t=1}^{N_{\text{used}_j}}$, we compute the posterior predictive distribution at the target fidelity $B$. The arms are then compared and pruned based on performances $y_{j}$. This allows the algorithm to distinguish early-saturating candidates from arms projected to achieve optimality at later stages.


\paragraph{Noise Model.} The observation variance $\Sigma$ in our theoretical analysis is the sum of the GP posterior predictive variances at the target fidelity $B$, updated dynamically at each round. The likelihood variance 
is treated as a GP hyperparameter, initialized to represent the stochasticity of the training process and estimated via marginal likelihood optimization \citep{rasmussen-book06a}. We assume a fixed training configuration in which increasing fidelity corresponds solely to increasing the number of training epochs, under which a stationary Gaussian noise model is a reasonable approximation for the observation stochasticity. In settings where fidelity changes induce significant shifts in the noise structure, e.g., data subsampling, 
a non-stationary variance schedule is required.

\section{Theoretical Guarantees}\label{sec:theoretical-guarantees}

In this section, we analyze the sample complexity and error probabilities of Prior-Guided Successive Halving (Algorithm \ref{alg:mf_prior_sh}, PSH). We first derive a bound on the expected Bayesian risk under $\epsilon$-optimality, showing how the prior $\mathcal{H}_0$ exponentially suppresses misidentification (Theorem \ref{theorem:expected probability}). Corollary \ref{cor:stopping-condition} isolates the budget requirement based on the expected risk. Crucially, we establish the worst-case sample complexity required to identify an $\epsilon$-best arm, highlighting how prior quality reduces the total budget $N$ (Theorem \ref{theorem worst case}).

\paragraph{Effective Budget Scaling}
In SH, the number of samples per arm varies by round. Let $n_r$ be the number of evaluations allocated to each surviving arm in round $r$. Given total budget~$N$, we have $n_r = \lfloor \frac{N}{R \cdot |S_r|} \rfloor$.
The posterior variance reduction is proportional to the effective sample count.
For arm $j$, define $\Delta_{\epsilon,j} := \max\{\epsilon, \mu_{j^*} - \mu_j\}$.
Adopting the proof strategy of \citet{atsidakou-arxiv22a}, we analyze $\epsilon$-BAI, incorporate multi-fidelity with GPs, derive a stopping rule, and link prior gaps to the budget.
By integrating the instance-dependent error only over the domain $\{ \boldsymbol{\mu} : \mu_{j^*} - \mu_j > \epsilon \}$, we derive an exponential error bound. The proof is provided in Appendix~\ref{subsec:appendix expected behavior}.

\begin{assumption}[GP Fidelity Approximation]\label{ass:gp}
Let $\hat{\mu}_{j,r}$ denote the GP posterior mean for arm $j$ at target fidelity $B$, conditioned on all observations collected up to round $r$, 
$\mathcal{D}_j = \{(t, y_{j,t})\}_{t=1}^{N_{\text{used}_j}}$, where $N_{\text{used}_j}$ denotes the cumulative number of observations per arm at round $r$. We assume there exists a function $\xi : \mathbb{N} \rightarrow \mathbb{R}_{>0}$ such that for all $j \in \mathcal{K}$ and all rounds $r$:
\begin{equation}
    \mathbb{P}\big(|\hat{\mu}_{j,r} - \mu_j| \leq \xi(n_r)\big) \geq 1 - \delta
\end{equation}
and $\xi(n_r) \rightarrow 0$ as $n_r \rightarrow \infty$. For both kernels used in our experiments, this holds by \citet{lederer2019uniform} 
(Theorem~3.3) under Lipschitz continuity of the kernel and the true learning curves on $[1, B]$, provided the GP posterior standard deviation 
satisfies $\sigma_{N_\text{used}}(t) \in \mathcal{O}(\log(n_r)^{-1/2 - \varepsilon})$ for some $\varepsilon > 0$. This condition is fulfilled by the linear kernel 
on LCBench at the faster parametric rate $\mathcal{O}(1/n_r)$~\cite{rasmussen-book06a}, and by the saturating exponential and RBF kernels on the synthetic benchmark at rates $\mathcal{O}(1/n_r)$ and $o(n_r^{-\alpha})$ for all $\alpha > 0$ respectively \citep{srinivas-icml10a} (see Appendix~\ref{subsection:worst case appendix}).
\end{assumption}

\begin{theorem}[$\epsilon$-Best Arm Identification]\label{theorem:expected probability}
Consider the PSH (Algorithm \ref{alg:mf_prior_sh}) with budget $N$.
Let $j^*$ denote the optimal arm. We define the error event $\mathcal{E}_\epsilon$ as returning an arm $j$ such that $\mu_{j^*} - \mu_j > \epsilon$.
The expected probability of error is bounded by summing the risk over all $R$ rounds:
\begin{align}
    \mathbb{E}[\mathbb{P}(\mathcal{E}_\epsilon \mid \boldsymbol{\mu})] \leq \sum_{r=0}^{R-1} C_{r, \epsilon} \sum_{j \in \mathcal{K}, j \neq j^*} \exp \left( - \frac{(\nu_{j^*} - \nu_j)^2}{4\sigma_0^2} \right) \exp \left( - \frac{n_r \epsilon^2}{4 \Sigma} \right)
\end{align}
\end{theorem}
where $C_{r,\varepsilon} = \frac{1}{2A_r\varepsilon\sqrt{4\pi\sigma_0^2}}$ with 
$A_r = \frac{n_r}{4\Sigma} + \frac{1}{4\sigma_0^2}$, combining 
the sampling noise and prior variance into a single precision term.
\begin{corollary}[Prior Stopping Condition]
\label{cor:stopping-condition}
Let $S_{r+1}$ be the set of arms retained after round $r$. To ensure the cumulative error probability remains below $\delta$, the algorithm can safely terminate at the end of round $r$ if the current budget satisfies:
\begin{equation}
\begin{split}
    N \geq \frac{4 \Sigma}{\epsilon^2} \bigg[  \ln \left( \frac{R}{\delta} \right) + \ln(C_{r, \epsilon}) 
     + \ln \left( \sum_{\substack{j \in S_{r+1}  j \neq j^*}} \exp \left( - \frac{(\nu_{j^*} - \nu_j)^2}{4\sigma_0^2} \right) \right) \bigg],
\end{split}
\end{equation}
where $j^*$ is the candidate best arm in $S_{r+1}$. This condition implies that sufficient samples have been collected to distinguish the candidate arm from the remaining set of survivors with high probability, given the prior separation.
\end{corollary}
Corollary~\ref{cor:stopping-condition} provides an expected-risk stopping criterion derived from Theorem~\ref{theorem:expected probability}: 
termination is safe in expectation once the prior separation and sampling evidence jointly drive the cumulative error below $\delta$. 
Theorem~\ref{theorem worst case} below strengthens this to a worst-case guarantee under Assumption~\ref{ass:gp}, deriving the explicit budget threshold $N_{\mathrm{stop}}$ used in Algorithm~\ref{alg:mf_prior_sh}. The proof is given in Appendix~\ref{subsection:worst case appendix}. 
%

Theorem~\ref{theorem worst case} guarantees that the GP approximation 
error eventually vanishes, which allows us to formally incorporate the 
uncertainty introduced by GP-based extrapolation across fidelities into our 
theoretical framework and establish that PSH can safely rely on GP predictions 
to move between fidelities without violating the worst-case guarantee of 
Theorem~\ref{theorem worst case}.
\begin{theorem}[Worst-case probability bound]\label{theorem worst case}
Under Assumption~\ref{ass:gp}, PSH (Algorithm~\ref{alg:mf_prior_sh}) used with an overall budget of $N_{\text{used}}$ 
identifies an $\epsilon$-best arm with probability at least 
$(1-\delta)$ if the following two conditions 
hold:
\begin{enumerate}[label=(\roman*)]
    \item \textbf{BAI condition:}
    \begin{equation}
        N_\text{used} \geq N_\text{stop} :=  \max_{j \in \mathcal{K} \setminus \{j^*\}} 
        \frac{4R\Sigma}{\tilde{\Delta}_{j,r}^2} \cdot 
        \left[\ln\!\left(\frac{2R(\frac{K}{2}-1)}{\delta}\right) 
        - \frac{(\nu_{j^*} - \nu_j)\tilde{\Delta}_{j,r}}{2\sigma_0^2}\right]
    \end{equation}
    \item \textbf{GP accuracy condition:} The used budget$N_\text{used}$ is large enough that 
    $\xi(N_\text{used}) \leq \epsilon/4$, i.e.,
    \begin{equation}
        N_\text{used} \geq N_{\mathrm{GP}} := \min\left\{n \in \mathbb{N} : 
        \xi(n) \leq \frac{\epsilon}{4}\right\}
    \end{equation}
    Such a finite $N_{\mathrm{GP}}$ is guaranteed to exist by 
    Assumption~\ref{ass:gp} and \citep{lederer2019uniform} (Theorem~3.3).
\end{enumerate}
\end{theorem}

\color{black}
\section{Comparison of Non-Stochastic and Prior-Guided Successive Halving}
Our derivation of prior-guided error bounds provides an extension to the non-stochastic Successive Halving (SH) algorithm established by \citet{jamieson-aistats16a}. While SH is minimax optimal for the prior-free setting, our results characterize the efficiency gains possible when prior knowledge is available. Specifically, we focus on the budget analysis provided by \citet{li-jmlr18a} in the context of the Hyperband algorithm.
In Theorem \ref{theorem worst case}, we provide a dynamic stopping condition that enables early termination when the combined evidence of the prior and observations distinguishes an $\epsilon$-optimal arm with sufficient distance to the others.

\subsection{From Instance-Dependent to Prior-Dependent Complexity}
Standard SH makes no distributional assumptions, instead relying on a deterministic envelope $\gamma(t)$ that bounds the error between the observed loss $\ell_{i,t}$ and the true limit value $\ell_{i,*}$ such that $|\ell_{i,t} - \ell_{i,*}| \le \gamma(t)$. \citet{li-jmlr18a} (Theorem 1) establish a sufficient budget for the identification of an $\epsilon$-optimal arm, if the used budget $N_\text{used}$ exceeds a threshold $N_\text{stop}$:
\begin{equation}
\begin{split}\label{equation:hyperband bound}
    N_\text{stop} = \,  2R 
      \max_{{j\in\mathcal{K}\setminus j^*}} j\,(1+ \gamma^{-1} \left( \max \left\{ \frac{\epsilon}{2}, \frac{\mu_{j^*} - \mu_j}{2} \right\} \right) .
\end{split}
\end{equation}
In this setting, $\gamma^{-1}(\cdot )$ represents the number of iterations required to reduce the error envelope below half the ground truth gap and $\eta$ is the elimination rate. With $\gamma^{-1}(y) \propto y^{-2}$ in the standard Gaussian HPO setting, this instance-dependent and frequentist bound requires a budget scaling with the sum of the inverse squared optimality gaps.
We consider the infinite-horizon complexity bound by \citet{li-jmlr18a}. Introducing a target confidence level $\delta$ allows us to invert the fixed-budget analysis: rather than minimizing an error for a fixed resource constraint, we calculate the sufficient budget required to guarantee PSH (Algorithm \ref{alg:mf_prior_sh}) finds an $\epsilon$-optimal arm, thereby enabling an early-stopping mechanism.

Crucially, our work introduces the prior mean $\nu_i$ as a component of the error bound. In the Gaussian observation model $\ell_{i,t} \sim \mathcal{N}(\mu_i, \Sigma/n_r)$, $\gamma$ is replaced by the standard error $\sqrt{\Sigma/n_r}$. Consequently, the frequentist sufficiency condition for SH requires that, in order to distinguish non-optimal from optimal arms, the sample complexity must satisfy
$n_{r} \ge \frac{4\Sigma}{(\mu_{j^*} - \mu_j)^2}.$
This corresponds to Equation \ref{equation:hyperband bound}. The standard SH algorithm must satisfy this for every arm, forcing the total budget to scale with the statistical difficulty of every arm.

In contrast, our Theorem \ref{theorem:expected probability} reshapes this worst case-counting problem into a prior-weighted concentration problem. The expected misidentification probability \mbox{$\mathbb{E}[\mathbb{P}(\mathcal{E}_\epsilon\mid \mu)]$} is determined by the concentration of the prior mass and the sampling budget:
\begin{equation}
\begin{split}
    \mathbb{E}[\mathbb{P}(\mathcal{E}_\epsilon \mid \mu)] \le \sum_{r=0}^{R-1} C_{r, \epsilon} \sum_{\substack{j \in S_r \\ j \neq j^*}} 
     \underbrace{\exp\left(-\frac{(\nu_{j^*} - \nu_j)^2}{4\sigma_0^2}\right)}_{\text{Prior Effect } (\mathcal{P})} 
     \cdot \underbrace{\exp\left(-\frac{n_r \epsilon^2}{4\Sigma}\right)}_{\text{Sampling Effect } (\mathcal{S})}.
\end{split}
\end{equation}
Crucially, this bound relaxes the $O((\mu_{j^*}-\mu_j)^{-2})$ sample complexity of the prior-free case. If the prior is informative (i.e., the prior gap $\nu_{j^*} - \nu_j$ is large), the prior effect term $\mathcal{P}$ decays to zero, effectively diminishing the influence of suboptimal arms from the sum. This allows the algorithm to identify an $\epsilon$-best arm, even when the sampling budget is insufficient to satisfy the frequentist condition from standard SH. Conversely, as $\sigma_0^2 \to \infty$, $\mathcal{P} \to 1$ for uninformative priors. We are consistent in that we recover the standard convergence rates of \citet{jamieson-aistats16a}. Appendix \ref{appendix::example} includes an example of the derivation of our bound.
A sufficient condition for $N_\text{stop} < N_{SH}$ is derived in Theorem \ref{theorem worst case}, which isolates the minimum prior gap required for budget savings.

\subsection{Robustness of Prior-Guided Successive Halving}
\label{subsec:robustnes}
While informative priors accelerate PSH, practical use cases may involve 
uncertain or misleading beliefs. For uninformative ($\nu_j=\nu$) priors, the prior term vanishes from the complexity bound.
We recover standard SH behavior. With misleading priors ($\nu_j > \nu_{j^*}$), the negative prior gap imposes a bounded budget penalty scaling linearly with the magnitude of error, but PSH cannot exceed the pre-defined budget $N$. Full details are provided in Appendix~\ref{subsec:robustnes_appendix}.

\section{Empirical Results}\label{sec:emipirical_results}

The following experiments are designed as proof-of-concept validation of our theoretical results rather than a comprehensive empirical benchmark. Our primary goal is to confirm that the budget reductions by Theorem 6.4 are observed in practice, and that the robustness guarantees against misleading priors hold empirically. To this end, we compare Prior-Guided Successive Halving (PSH) to standard Successive Halving (SH). We consider five prior construction strategies ranging from highly informative to intentionally misleading, including rank-based, performance-based, indicator, uniform, and inverse rank priors; details are provided in Appendix~\ref{sec:appendix prior generation}. In Section~\ref{subsec:experimental setup} we describe our experimental setup and discuss the results in Section~\ref{subsec:results}.

\subsection{Experimental Setup}\label{subsec:experimental setup}
To assess the optimization performance of PSH, we conduct experiments on a synthetic benchmark, as well as a standard HPO benchmark. 
Details can be found in Appendix~\ref{sec:appendix experimental setup}.
We construct a \textbf{synthetic search space} where each arm follows a smooth, saturating learning curve. The target function is $f_j(t) = \mu_j \bigl(1 - e^{-t/\tau_j}\bigr)$ at each fidelity $t\leq B=256$. We sample the true final performance $\mu_j\in[0,1]$ for each of the 256 candidate arms, while $\tau_j$ varies linearly across arms. Thus, we induce bias from observing arms at low fidelities in lieu of stochastic noise. 
We use a Gaussian process with a saturating exponential kernel plus an RBF adjustment~\citep{rasmussen-book06a}, yielding a flexible non-stationary kernel for the promotion across fidelities. 
Figure \ref{fig:learning curves} in Appendix \ref{subsec:appendix benchmark} shows example learning curves for varying $\mu$ and $\tau$.
We evaluate on \textbf{LCBench}~\citep{zimmer-tpami21a} via YAHPO Gym~\citep{pfisterer-automl22a}, using $K=256$ arms and maximum fidelity $B=52$. For this benchmark, we model learning curves with a linear kernel, as discussed in Section \ref{sec:prior guided sh}. Further experimental results are provided in \cref{sec:appendix experimental setup}.
All experiments are conducted on compute nodes equipped with AMD EPYC Milan 7763 (2.45GHz), 256GB RAM CPUs of which 2 cores and 4GB RAM are allocated per run.
\begin{figure}[t]
    \centering
    \begin{subfigure}[t]{0.4\linewidth}
        \centering
        \texttt{Synthetic Results}
        \includegraphics[width=\linewidth,trim={0cm 0.25cm 0cm 0cm},clip]{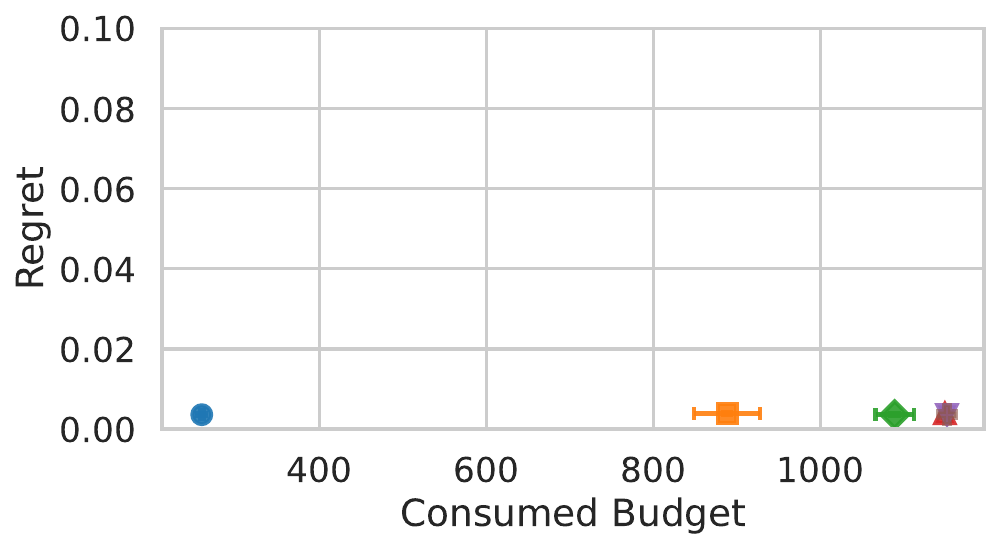}
        \label{fig:synthetic}
    \end{subfigure}
    \hfill
    \begin{subfigure}[t]{0.4\linewidth}
        \centering
        \texttt{LCBench Results}
        \includegraphics[width=\linewidth,trim={0cm 0.25cm 0cm 0cm},clip]{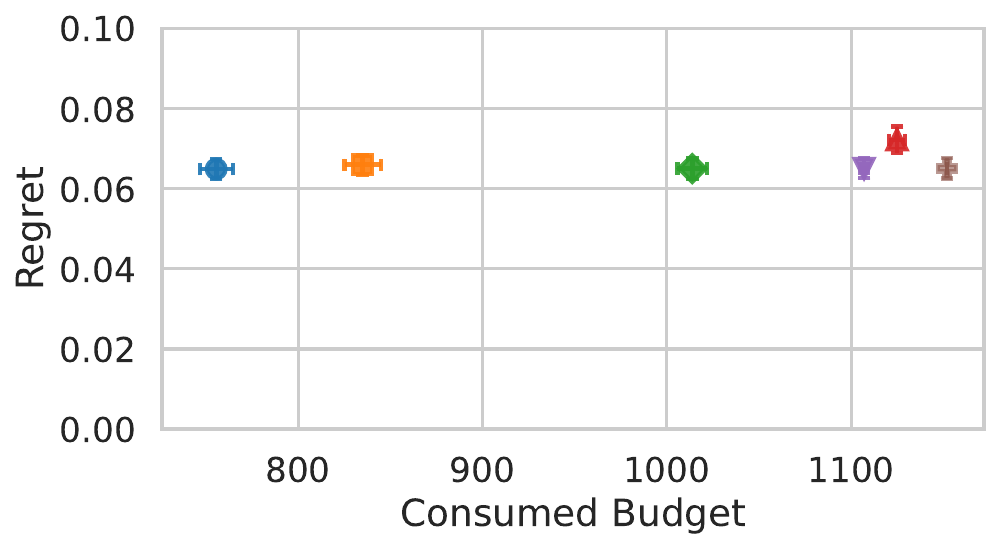}
        \label{fig:lcbench}
    \end{subfigure}
    \includegraphics[width=\linewidth,trim={0cm 0cm 0cm 0.25cm},clip]{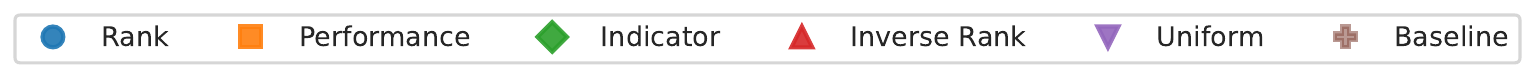}
    \caption{Performance of prior-guided successive halving (PSH) with different priors on the synthetic benchmark (left) and LCBench (right), showing consumed budget and validation regret with standard deviations. On the left, the results for uniform and inverse rank prior overlap with Standard SH.}
    \label{fig:empirical_resultsparte_front_lcbench}
\end{figure}

\subsection{Results}\label{subsec:results}
Figure \ref{fig:empirical_resultsparte_front_lcbench} reports the mean regret and consumed budget, with standard deviations given, for both the synthetic benchmark and LCBench. Across both settings, the results show a consistent pattern: PSH achieves substantial budget savings while maintaining minimal regret. Below, we discuss the effect of different prior kinds in order of decreasing effect on the early stopping criterion.

For \textbf{rank priors}, PSH achieves the most substantial budget reductions on both benchmarks. On the synthetic benchmark, they lead to the largest budget reductions at $256$ evaluations compared to the standard SH budget of $1150$, implying an evaluation budget reduction of roughly $80\%$.
On LCBench, the rank prior also results in the largest effect of the early stopping criterion. After an average budget of around $750$, the criterion triggers, which is a reduction of roughly $35\%$. The regret remains almost identical to standard SH, indicating that the rank prior is highly informative.
For the \textbf{performance prior}, we observe similar behavior to the rank prior. On the synthetic benchmark, it achieves near-zero regret with a substantially reduced budget.

On LCBench, it stops after approximately $840$ evaluations. It maintains regret comparable to standard SH, indicating that performance-based information offers a considerable speedup.
The \textbf{indicator prior} still reduces the budget compared to standard SH, ranking third regarding consumed budget. 

With the \textbf{uniform prior}, we observe only a marginal budget reduction compared to standard SH. As discussed in Section \ref{subsec:robustnes}, the prior term vanishes in the early-stopping criterion for uniform priors; however, as evaluations increase, the shrinking GP posterior variance $\Sigma$ and widening predicted gaps $\tilde{\Delta}_{j,r}$ between arms still slightly reduce the required remaining budget. For both benchmarks, this reduction is negligible.
Notably, PSH recovers from the misleading \textbf{inverse rank prior} without significant budget inflation or regret, demonstrating robustness: the GP-based model quickly overrules incorrect priors as observations accumulate, reverting the stopping rule to standard SH behavior. Furthermore, since PSH executes a strict prefix of the evaluations performed by standard SH, any reduction in consumed budget translates directly and provably to a reduction in energy consumption.

A \textbf{sensitivity analysis} regarding user confidence $\sigma_0$ is provided in Appendix~\ref{sec:appendix further results}.
Furthermore, our stopping criterion transfers to Hyperband and PriorBand, achieving consistent budget savings without compromising solution quality. This demonstrates its compatibility as an early-stopping mechanism beyond standard SH (see Appendices~\ref{sec:appendix:hyperband} and~\ref{sec:appendix:priorband}).
Overall, our results show that informative priors reduce sample complexity while maintaining regret comparable to the baseline.

\section{Conclusion}
We presented a theoretical analysis of how priors govern sample complexity in fixed-budget $\epsilon$–BAI. Distribution-dependent bounds show that informative priors provably yield budget reductions, while being robust to uninformative or misleading priors.
Building on these insights, we introduced PSH, a multi-fidelity, prior-guided variant of SH that integrates GP models with a theoretically grounded early-stopping rule, achieving significant budget reductions without compromising regret.
Together, these results provide a principled foundation for efficient and more sustainable HPO.

\paragraph{Limitations \& Broader Impact} The degree of budget savings achieved by PSH depends directly on the 
quality and informativeness of the provided prior. While our theoretical 
analysis characterizes robustness to uninformative and misleading priors, 
the practical benefit of PSH is tied to the availability of meaningful 
prior knowledge, whose origin and reliability may vary across use cases 
and application domains. Our worst-case guarantees (Theorem~\ref{theorem worst case}), specifically the convergence of the GP estimation to the ground truth, rely on Assumption~\ref{ass:gp}, requiring the standard deviation of the GP $\sigma_{N_\text{used}}(t) \in \mathcal{O}(\log(n_r)^{-1/2 - \varepsilon})$, which may fail under kernel misspecification, non-stationary observation noise from data subsampling, or significant distribution shifts across fidelities. In such cases, biased GP extrapolation could prevent $\tilde{\Delta}_{j,r}$ from converging fast enough to satisfy Theorem~\ref{theorem worst case} within a practical budget. A promising way to relax this is by replacing GP surrogates with learning-curve foundation models~\citep{rakotoarison-icml24a}, which could provide robust extrapolation and strengthen PSH's early-stopping decisions.
This paper advances the methodology of HPO. We do not foresee specific societal risks beyond those generic to improved optimization techniques.



\bibliographystyle{abbrvnat}
\bibliography{strings,lib,ext_lib,proc}

\newpage
\appendix
\onecolumn
\section*{Overview of the Appendix}
\contentsline {section}{\numberline{\ref{list of symbols}} List of Symbols}{\pageref{list of symbols}}{}

\contentsline {section}{\numberline{\ref{sec:appendix theoretical guarantees}} Theoretical Guarantees}{\pageref{sec:appendix theoretical guarantees}}{}
\contentsline {subsection}{\numberline{\ref{subsec:appendix expected behavior}} Expected behavior of Prior-Guided Successive Halving}{\pageref{subsec:appendix expected behavior}}{}
\contentsline {subsection}{\numberline{\ref{subsection:worst case appendix}} Worst Case Probability Bound for Prior-Guided Successive Halving}{\pageref{subsection:worst case appendix}}{}
\contentsline {subsection}{\numberline{\ref{appendix:comparison of sh bounds}} Comparison of Prior-Guided Successive Halving to Standard Successive Halving}{\pageref{appendix:comparison of sh bounds}}{}
\contentsline {subsection}{\numberline{\ref{appendix::example}} Example}{\pageref{appendix::example}}{}
\contentsline {subsection}{\numberline{\ref{appendix:prior efficiency}} Prior Efficiency Condition}{\pageref{appendix:prior efficiency}}{}

\contentsline {section}{\numberline{\ref{sec:appendix experimental setup}} Details on the Experimental Setup and Prior Generation}{\pageref{sec:appendix experimental setup}}{}
\contentsline {subsection}{\numberline{\ref{subsec:appendix benchmark}} Benchmarks}{\pageref{subsec:appendix benchmark}}{}
\contentsline {subsection}{\numberline{\ref{sec:appendix prior generation}} Prior Generation}{\pageref{sec:appendix prior generation}}{}
\contentsline {subsection}{\numberline{\ref{subsec:robustnes_appendix}} Robustness of Prior-Guided Successive Halving}{\pageref{subsec:robustnes_appendix}}{}

\contentsline {section}{\numberline{\ref{sec:appendix algorithm}} Multi-Fidelity Prior-Guided Successive Halving (PSH) Algorithm}{\pageref{sec:appendix algorithm}}{}
\contentsline {subsection}{\numberline{\ref{sec:appendix:hyperband}} Hyperband Extension}{\pageref{sec:appendix:hyperband}}{}
\contentsline {subsection}{\numberline{\ref{sec:appendix:priorband}} PriorBand Extension}{\pageref{sec:appendix:priorband}}{}

\contentsline {section}{\numberline{\ref{sec:appendix min prior gap plot}} Analysis of the Minimum Required Prior Gap}{\pageref{sec:appendix min prior gap plot}}{}

\contentsline {section}{\numberline{\ref{sec:appendix further results}} Further Results}{\pageref{sec:appendix further results}}{}
\contentsline {subsection}{\numberline{\ref{subsec:appendix kernel ablation}} LCBench Ablation with Different Kernels}{\pageref{subsec:appendix kernel ablation}}{}
\contentsline {subsection}{\numberline{\ref{subsec:appendix sigma ablation}} Ablation for Varying Sigmas}{\pageref{subsec:appendix sigma ablation}}{}

\newpage
\section{List of Symbols}\label{list of symbols}
\begin{table}[H]
\centering
\caption{List of Symbols}
\label{tab:symbols}
\begin{tabular}{l p{10cm}}
\toprule
\textbf{Symbol} & \textbf{Description} \\
\midrule
\multicolumn{2}{l}{\textit{Problem Setup}} \\
$\mathcal{K}$ & Set of arms \\
$K$ & Total number of arms \\
$L_i(b)$ & Learning curve function for arm $i$ at fidelity $b \in [0, B]$ \\
$B$ & Maximum fidelity \\
$\mu_i$ & True mean reward of arm $i$ at maximum fidelity, $\mu_i := L_i(B)$ \\
$j^*$ & Index of the optimal arm, $j^* \in \argmax_{k} \mu_k$ \\
$y_t$ & Noisy observation at step $t$, $y_t = L_{i_t}(b_t) + \eta_t$ \\
$\eta_t$ & Independent Gaussian noise, $\eta_t \sim \mathcal{N}(0,\sigma_{i_t}^2)$ \\
$\Sigma$ & Sum of GP posterior predictive variances at target fidelity $B$ \\
$N$ & Total fixed evaluation budget \\
$N_{\text{used}}$ & Total budget consumed so far across all arms and rounds \\
$N_{\text{used}_j}$ & $N_{\text{used}_j}$ per arm $j\in\mathcal{K}$\\
$\epsilon$ & Margin for $\epsilon$-Best Arm Identification \\
$\mathcal{E}_\epsilon$ & Error event: returned arm $j$ satisfies $\mu_{j^*} - \mu_j > \epsilon$ \\
$j$ & Arm returned by the algorithm \\
\midrule
\multicolumn{2}{l}{\textit{Priors and Beliefs}} \\
$\mathcal{H}_0$ & User-provided Gaussian Process (GP) prior belief over arm means \\
$\nu_i$ & Prior mean of arm $i$ \\
$\sigma_0^2$ & Prior variance; smaller values indicate higher user certainty \\
$\nu_{ij}$ & Difference in prior means, $\nu_{ij} = \nu_{j^*} - \nu_j$ \\
$\hat{\mu}_{j,r}$ & GP posterior mean for arm $j$ at target fidelity $B$ in round $r$ \\
$\sigma_{j,r}^2$ & GP posterior predictive variance for arm $j$ at target fidelity $B$ in round $r$ \\
$\mathcal{D}_j$ & Observations collected for arm $j$, $\mathcal{D}_j = \{(t, y_{j,t})\}_{t=1}^{N_{\text{used}_j}}$ \\
\midrule
\multicolumn{2}{l}{\textit{Successive Halving and Analysis}} \\
$R$ & Total number of rounds, $R = \lceil \log_\eta K \rceil$ \\
$r$ & Current round index, $r \in \{0, \dots, R-1\}$ \\
$\eta$ & Elimination rate; set to $\eta = 2$ throughout \\
$S_r$ & Set of active arms remaining in round $r$ \\
$n_r$ & Target total observations per arm at end of round $r$, $n_r = \lfloor N / (R \cdot |S_r|) \rfloor$; new observations per arm in round $r$ are $n_r - n_{r-1}$ \\
$n_{\text{prev}}$ & Value of $n_r$ at the end of the previous round \\
$\Delta_{\epsilon,j}$ & True performance gap, $\Delta_{\epsilon,j} = \max\{\epsilon, \mu_{j^*} - \mu_j\}$ \\
$\tilde{\Delta}_{j,r}$ & GP-adjusted gap, $\tilde{\Delta}_{j,r} = \max\{\epsilon, \hat{\mu}_{j^*,r} - \hat{\mu}_{j,r} - 2\xi(n_r)\}$ \\
$g_r$ & Biased gap function quantifying identification difficulty in round $r$ \\
$A_r$ & Precision term, $A_r = \frac{n_r}{4\Sigma} + \frac{1}{4\sigma_0^2}$ \\
$C_{r,\epsilon}$ & Bounding coefficient, $C_{r,\varepsilon} = \frac{1}{2A_r\varepsilon\sqrt{4\pi\sigma_0^2}}$ \\
$\delta$ & Confidence parameter; algorithm succeeds with probability at least $1 - \delta$ \\
$\delta_{\text{GP}}$ & GP approximation failure probability; set to $\delta / (4RK)$ \\
$N_{\text{stop}}$ & Sufficient budget threshold derived from Theorem~\ref{theorem worst case}, condition~(i) \\
$N_{\text{GP}}$ & Minimum budget for GP accuracy condition, $N_{\text{GP}} = \min\{n : \xi(n) \leq \epsilon/4\}$ \\
$\xi(n_r)$ & GP approximation error bound; $\mathbb{P}(|\hat{\mu}_{j,r} - \mu_j| \leq \xi(n_r)) \geq 1 - \delta_{\text{GP}}$ \\
$\sigma_{n_r}(t)$ & GP posterior standard deviation at fidelity $t$ given $n_r$ observations \\
\bottomrule
\end{tabular}
\end{table}
\newpage
\section{Theoretical Guarantees}\label{sec:appendix theoretical guarantees}
\subsection{Expected behavior of Prior-Guided Successive Halving}\label{subsec:appendix expected behavior}

\begin{theorem}[$\epsilon$-Best Arm Identification]\label{theorem:expected probability_appendix}
Consider the Successive Halving algorithm with budget $N$, where the optimization is initialized with a prior $\mathcal{N}(\nu_{i}, \sigma_{0}^2)$ and arms are eliminated based on their posterior means (Algorithm \ref{alg:mf_prior_sh}). Let $j^*$ denote the optimal arm. We define the error event $\mathcal{E}_\epsilon$ as returning an arm $j$ such that $\mu_{j^*} - \mu_j > \epsilon$.
The expected probability of error is bounded by summing the risk over all $R$ rounds:
\begin{align}
    \mathbb{E}[\mathbb{P}(\mathcal{E}_\epsilon \mid \boldsymbol{\mu})]
    \leq \sum_{r=0}^{R-1} C_{r, \epsilon} \sum_{j \in \mathcal{K}, j \neq j^*} \exp \left( - \frac{(\nu_{j^*} - \nu_j)^2}{4\sigma_0^2} \right) \exp \left( - \frac{n_r \epsilon^2}{4 \Sigma} \right)
\end{align}

\end{theorem}

\begin{proof}
The algorithm proceeds in $R$ rounds. An error occurs if the optimal arm $j^*$ is eliminated by a suboptimal arm $j \in S_r$ based on their posterior means after collecting $n_r$ samples.

\paragraph{The Biased Gap (Fixed Instance Analysis)}
Fix the bandit instance $\boldsymbol{\mu}$. Following Lemma 3 of \citet{atsidakou-arxiv22a}, the probability that a suboptimal arm $j$ eliminates the optimal arm $j^*$ in round $r$ (using budget $n_r$) is bounded by the exponential of the biased gap function $g_r$:
\begin{equation}
    \mathbb{P}(j\in S{_r+1},  j^*\notin S_{r+1} \mid \boldsymbol{\mu}) \leq \exp \left( - g_r(j^*, j, \boldsymbol{\mu}) \right)
\end{equation}
where the biased gap depends on the round-specific budget $n_r$:
\begin{equation}
    g_r(j^*, j, \boldsymbol{\mu}) = \frac{n_r (\mu_{j^*} - \mu_j)^2}{4 \Sigma} + \frac{(\nu_{j^*} - \nu_j)(\mu_{j^*} - \mu_j)}{2\sigma_0^2}.
\end{equation}

\paragraph{Bayesian Integration Setup}
We calculate the expected risk by integrating the instance-dependent probability over the prior distribution. We restrict our attention to the region where the suboptimality gap exceeds $\epsilon$. Let $\Delta = \mu_{j^*} - \mu_j$. We integrate over the marginal prior for the gap, which is distributed as $\Delta \sim \mathcal{N}(\nu_{ij}, 2\sigma_0^2)$, where $\nu_{ij} = \nu_{j^*} - \nu_j$.
The instance-averaged probability of error for a specific pair $(j^*, j)$ is
\begin{equation}
    I_{ij}(\epsilon) = \int_{\epsilon}^{\infty} \frac{1}{\sqrt{4\pi\sigma_0^2}}\exp \left( - g_r(\Delta) \right)  \exp\left( - \frac{(\Delta - \nu_{ij})^2}{4\sigma_0^2} \right) d\Delta.
\end{equation}

\paragraph{Completing the Square}
We analyze the combined exponent of the integrand to solve the integral. Let $E(\Delta)$ denote the sum of the exponents from the likelihood (biased gap) and the prior distribution.
From Step 1, the likelihood exponent is $- \left( \frac{n_r \Delta^2}{4 \Sigma} + \frac{\nu_{ij}\Delta}{2\sigma_0^2} \right)$.
Summing this with the prior exponent $- \frac{(\Delta - \nu_{ij})^2}{4\sigma_0^2}$, we obtain:
\begin{equation}
    E(\Delta) = - \frac{n_r \Delta^2}{4 \Sigma} - \frac{\nu_{ij}\Delta}{2\sigma_0^2} - \frac{(\Delta - \nu_{ij})^2}{4\sigma_0^2}.
\end{equation}
Expanding the square in the prior term, the linear cross-terms involving $\Delta$ cancel out exactly. We group the remaining $\Delta^2$ terms and the constant term:
\begin{equation}
    E(\Delta) = - \Delta^2 \underbrace{\left( \frac{n_r}{4 \Sigma} + \frac{1}{4\sigma_0^2} \right)}_{A_r} - \frac{\nu_{ij}^2}{4\sigma_0^2}.
\end{equation}
This factorization allows us to pull the term independent of $\Delta$ (the prior separation) outside the integral:
\begin{equation}
    I_{ij}(\epsilon) = \frac{1}{\sqrt{4\pi\sigma_0^2}} 
\exp \left( - \frac{\nu_{ij}^2}{4\sigma_0^2} \right) 
\int_{\epsilon}^{\infty} \exp \left( - A_r \Delta^2 \right) d\Delta.
\end{equation}
We apply the Gaussian tail bound inequality $\int_{x}^{\infty} e^{-kt^2} dt \le \frac{1}{2kx} e^{-kx^2}$ for $x>0$. With $k=A_r$ and $x=\epsilon$, we have:
\begin{equation}
    \int_{\epsilon}^\infty e^{-A_r \Delta^2} d\Delta \leq \frac{1}{2 A_r \epsilon} \exp(-A_r \epsilon^2).
\end{equation}
Substituting the definition of $A_r$ into the exponent yields $\exp(-A_r \epsilon^2) = \exp\left( - \frac{n_r \epsilon^2}{4\Sigma} \right) \exp\left( - \frac{\epsilon^2}{4\sigma_0^2} \right)$.
Dropping the second exponential term (as it is $\le 1$) and defining the coefficient $C_r = \frac{1}{2 A_r \epsilon}$, we obtain the final bound for the pair $(j^*, j)$:
\begin{equation}
    I_{ij}(\epsilon) \leq C_{r, \epsilon} \cdot \exp \left( - \frac{(\nu_{j^*} - \nu_j)^2}{4\sigma_0^2} \right) \exp \left( - \frac{n_r \epsilon^2}{4 \Sigma} \right).
\end{equation}

\paragraph{Conclusion (Union Bound over Rounds)}
Unlike the fixed budget setting where $n$ is constant, here $n_r$ grows with $r$. The total error probability is the sum of error probabilities across all rounds $r=0 \dots R-1$. In each round, we sum over the active suboptimal arms:
\begin{equation}
    \mathbb{E}[\text{Error}] \leq \sum_{r=0}^{R-1} \sum_{j \in \mathcal{K}} I_{ij}(\epsilon).
\end{equation}
Substituting the bound for $I_{ij}(\epsilon)$ yields the theorem.

\end{proof}

\newpage
\subsection{Worst Case Probability Bound for Prior-Guided Successive Halving}\label{subsection:worst case appendix}

\begin{theorem}[Worst case probability bound]
    Algorithm \ref{alg:mf_prior_sh} identifies an $\epsilon$-best arm with probability at least $1-\delta$ if the overall budget $N$ fulfills
        \begin{equation}
           N \geq \max_{j \in \mathcal{K}\setminus\{j^*\}} \frac{4R\Sigma}{\max\{\epsilon, \mu_{j^*}-\mu_j\}^2} 
            \cdot \Bigg[ \ln\left(\frac{2\log_2(K)(\frac{K}{2}-1)}{\delta}\right) 
            - \frac{(\nu_{j^*}-\nu_j)\max\{\epsilon, \mu_{j^*}-\mu_j\}}{2\sigma_0^2} \Bigg].
        \end{equation}      
\end{theorem}

\begin{proof}
    \underline{Case 1:} 
Let for all $j\in \mathcal{K}\backslash\{j^*\}$: $\mu_{j^*}-\mu_j \geq \epsilon$.\\
    Then we have 
    \begin{equation}\label{Eq:BudgetCase1}
        N \geq  \max_{j \in \mathcal{K}\setminus\{j^*\}} \frac{4R\Sigma}{(\mu_{j^*}-\mu_j)^2} 
         \cdot \Bigg[ \ln\left(\frac{2\log_2(K)(\frac{K}{2}-1)}{\delta}\right) \\
        - \frac{(\nu_{j^*}-\nu_j)(\mu_{j^*}-\mu_j)}{2\sigma_0^2} \Bigg] = (*).
    \end{equation}
    Let us fix in the following arm $j \in \mathcal{K}\backslash\{j^*\}$ that maximizes $(*)$ as $j_{max} \in \argmax_{j \in \mathcal{K}\backslash\{j^*\}}(*)$. Then Equation $\ref{Eq:BudgetCase1}$ is equivalent to
    \begin{align*}
    & N\frac{(\mu_{j^*}-\mu_{j_{\max}})^2}{4R\Sigma} \\
    &\quad \geq \ln\!\left(\frac{2\log_2(K)(\frac{K}{2}-1)}{\delta}\right) - \frac{(\nu_{j^*}-\nu_{j_{\max}})(\mu_{j^*}-\mu_{j_{\max}})}{2\sigma_0^2} \\[6pt]
    \Leftrightarrow\quad & {-}N\frac{(\mu_{j^*}-\mu_{j_{\max}})^2}{4R\Sigma} - \frac{(\nu_{j^*}-\nu_{j_{\max}})(\mu_{j^*}-\mu_{j_{\max}})}{2\sigma_0^2} \\
    &\quad \leq \ln\!\left(\frac{\delta}{2\log_2(K)(\frac{K}{2}-1)}\right) \\[6pt]
    \Leftrightarrow\quad & 2\exp\!\Bigg({-}N\frac{(\mu_{j^*}-\mu_{j_{\max}})^2}{4R\Sigma} - \frac{(\nu_{j^*}-\nu_{j_{\max}})(\mu_{j^*}-\mu_{j_{\max}})}{2\sigma_0^2}\Bigg) \\
    &\quad \leq \frac{\delta}{\log_2(K)(\frac{K}{2}-1)}
    \end{align*}
    Applying Lemma 4 from \citet{atsidakou-arxiv22a} leads to
    \begin{align*}
        \mathbb{P}\left(j^* \notin S_{r+1} | j^* \in S_{r}, \mathbf{\mu}\right) \leq \frac{\delta}{\log_2(K)(\frac{K}{2}-1)}.
    \end{align*}
    By the union bound over all rounds $r \in \{1,\dots,\log_2(K)\}$, for $K\geq 4$, we get an overall worst case bound for the misidentification probability of 
    \begin{align*}
        \mathbb{P}(j^* \notin S_R \mid \mu) &= \mathbb{P}\!\left(\bigcup_{r=1}^{R} j^* \in S_{r-1},\, 
        j^* \notin S_r \,\middle|\, \mu\right)
        \\&\leq \sum_{r=1}^{R} \mathbb{P}(j^* \notin S_r 
        \mid j^* \in S_{r-1}, \mu)\mathbb{P}(j^* \in S_{r-1} \mid \mu) \\&\leq \sum_{r=1}^{R} 
        \mathbb{P}(j^* \notin S_r \mid j^* \in S_{r-1}, \mu)
        \\&\leq \frac{\delta}{(\frac{K}{2}-1)} \leq \delta
    \end{align*}
    \\
    \newpage
    \underline{Case 2:}

Let $\epsilon \geq \mu_{j^*} - \mu_j$ for some $j \in \mathcal{K}$.\\
    For all arms $\tilde{j} \in \mathcal{K}$ for which $\mu_{j^*}-\mu_{\tilde{j}} \geq \epsilon$ holds, we also have
    \begin{align*}
        N & \geq \max_{j \in \mathcal{K}\setminus\{j^*\}} \frac{4R\Sigma}{\epsilon^2} \Bigg[ \ln\left(\frac{2\log_2(K)(\frac{K}{2}-1)}{\delta}\right)  - \frac{(\nu_{j^*}-\nu_j)\epsilon}{2\sigma_0^2} \Bigg] \\
        & \geq \frac{4R\Sigma}{(\mu_{j^*} - \mu_{\tilde{j}})^2} \Bigg[ \ln\left(\frac{2\log_2(K)(\frac{K}{2}-1)}{\delta}\right)  - \frac{(\nu_{j^*}-\nu_{\tilde{j}})(\mu_{j^*}-\mu_{\tilde{j}})}{2\sigma_0^2} \Bigg] \\
        \Leftrightarrow ~ & \exp\Bigg(-N\frac{(\mu_{j^*}-\mu_{\tilde{j}})^2}{4R\Sigma}  - \frac{(\nu_{j^*}-\nu_{\tilde{j}})(\mu_{j^*}-\mu_{\tilde{j}})}{2\sigma_0^2}\Bigg)\\ 
        & \leq \frac{\delta}{2\log_2(K)(\frac{K}{2}-1)}.
    \end{align*}    
    Lemma 3 by \citet{atsidakou-arxiv22a} gives us $\mathbb{P}(\hat{\mu}_{\tilde{j},r} > \hat{\mu}_{j^*,r}) \leq \frac{\delta}{\log_2(K)(K-2)} $ for a fixed but random round $r \in \{1,\dots,\log_2(K)\}$. We know by the case assumption that there is at least one arm $j \in \mathcal{K}\backslash\{j^*\}$ for which $\epsilon \geq \mu_{j^*} - \mu_j$. Taking the union bound over all other arms in $\mathcal{K}\backslash\{j^*\}$ and over all rounds $r \in \{1,\dots,\log_2(K)\}$, we can bound the probability that one of these arms are estimated better than $j^*$ during the whole run of the algorithm by $\delta$. 
    For all other arms we have $\mu_{j^*}-\mu_j \leq \epsilon$ and thus $j$ is $\epsilon$-optimal and is considered as a valid return of our algorithm.
\end{proof}


\paragraph{Proof of condition~(ii).}
By Assumption~\ref{ass:gp}, the GP posterior standard deviation satisfies 
$\sigma_{n_r}(t) \in \mathcal{O}(\log(n_r)^{-1/2-\varepsilon})$ for some 
$\varepsilon > 0$ for both kernels used in our experiments. Under this 
condition, \citet{lederer2019uniform} (Theorem~3.3) guarantees that 
$\xi(n_r) \to 0$ as $n_r \to \infty$. Therefore there exists a finite
\begin{equation}
    n_{\mathrm{GP}} := \min\{n \in \mathbb{N} : \xi(n) \leq \epsilon/4\}
\end{equation}
such that for all $n_r \geq n_{\mathrm{GP}}$, the GP approximation error 
is negligible relative to $\epsilon$. The threshold $\epsilon/4$ follows 
directly from the GP-adjusted gap, defined as
\begin{equation}
    \tilde{\Delta}_{j,r} := \max\{\epsilon, \hat{\mu}_{j^*,r} - 
    \hat{\mu}_{j,r} - 2\xi(n_r)\},
\end{equation}
which deflates the observed gap by $2\xi(n_r)$ to account for the GP 
approximation error on both sides. Requiring $\xi(n_r) \leq \epsilon/4$ 
ensures $2\xi(n_r) \leq \epsilon/2$, so that
\begin{equation}
    \tilde{\Delta}_{j,r} \geq \epsilon - 2\xi(n_r) \geq \epsilon/2 > 0.
\end{equation}
Once $n_r \geq n_{\mathrm{GP}}$, the difference between the GP posterior 
mean and the true arm mean is at most $\epsilon/4$, making the GP 
predictions sufficiently accurate for our stopping criterion to safely be invoked.

\newpage
\subsection{Comparison of Prior-Guided Successive Halving to Standard Successive Halving}\label{appendix:comparison of sh bounds}
We establish the theoretical conditions under which our proposed prior-guided Successive Halving (PSH) results in sampling budget savings compared to the standard SH baseline. By equating the sufficient budget condition of the standard algorithm with the worst case complexity bound of our prior-guided approach in Theorem \ref{theorem worst case}, we derive the requirements for the quality of the prior distribution. Relative to the observation noise and instance hardness, we isolate the minimum prior gap $\nu_{j^*}-\nu_j$ necessary to guarantee reduction of the total evaluation budget due to the prior.

For readability, let $\Delta_{\epsilon,j} = \max\{\epsilon, \mu_{j^*}-\mu_j\}$ and $\tilde{\Delta}_{j,r} \leftarrow \max\{\epsilon, \hat{\mu}_{j^*,r} - \hat{\mu}_{j,r} - 2\xi(n_r)\}$.

Under Assumption~\ref{ass:gp} and condition~(ii) of
Theorem~\ref{theorem worst case}, $|\tilde{\Delta}_{j,r} - \Delta_{\epsilon,j}|
\leq 2\xi(n_r) \leq \epsilon/2$. The following analysis treats
$\tilde{\Delta}_{j,r}$ as the effective gap.

Hyperband requires the means to be sorted in ascending order by closeness to the optimum. Hence, let $\mu_{j^*}= \mu_1\geq\mu_2\geq...\geq\mu_{K}$.
\begin{align}
    N_{SH} &=\,  2R\max_{{k\in\mathcal{K}\setminus \{j^*\}}} k\,\left( 1+ \frac{\ln(\frac{2}{\delta})}{2\max \left\{ \frac{\epsilon}{2}, \frac{\mu_{j^*} - \mu_k}{2} \right\}^2}  \right) 
    \\&\geq \max_{j \in \mathcal{K}\setminus\{j^*\}} \frac{4R\Sigma}{\Delta_{\epsilon,j}^2} 
              \Bigg[ \ln\left(\frac{2R(\frac{K}{2}-1)}{\delta}\right) 
             - \frac{(\nu_{j^*}-\nu_j)\Delta_{\epsilon,j}}{2\sigma_0^2} \Bigg]
    \end{align}
Note that the arm maximizing the standard SH bound (LHS) is not necessarily the same arm maximizing our prior-guided bound (RHS). Since the sample complexity on the RHS scales with the negative gap $-\Delta_{\epsilon,j}$, the maximum occurs at the arm with the minimal gap to the optimum. By the above ordering of the arms, this is $\mu_2\in\arg\min_{j\in\mathcal{K}\setminus \{j^*\}} \Delta_{\epsilon,j}$. Since the budget $N_{SH}$ is defined as the maximum cost over all arms, it is larger than or equal to the budget for any given arm $j\in\mathcal{K}\setminus \{j^*\}$, especially also for $\mu_2$. Therefore, the inequality holds provided that the standard SH cost for $\mu_2$ is greater than the highest cost of our prior-guided approach. In the following, we fix $j=2$ and solve the inequality for the required prior gap.
\begin{align*}
        N_{SH} &\geq\,  2R \, j\,\left( 1+ \frac{\ln(\frac{2}{\delta})}{2\max \left\{ \frac{\epsilon}{2}, \frac{\mu_{j^*} - \mu_{j}}{2} \right\}^2}  \right) 
    &&\geq \frac{4R\Sigma}{\Delta_{\epsilon,j}^2} 
              \Bigg[ \ln\left(\frac{2R(\frac{K}{2}-1)}{\delta}\right) 
             - \frac{(\nu_{j^*}-\nu_{j})\Delta_{\epsilon,j}}{2\sigma_0^2} \Bigg]\\
    &\Leftrightarrow\quad   \frac{j\Delta_{\epsilon,j}^2}{2\Sigma}\left( 1+ \frac{2\ln(\frac{2}{\delta})}{\Delta_{\epsilon,j}^2}  \right) 
    &&\geq\quad \ln\left(\frac{2R(\frac{K}{2}-1)}{\delta}\right) 
             - \frac{(\nu_{j^*}-\nu_{j})\Delta_{\epsilon,j}}{2\sigma_0^2} \\
    &\Leftrightarrow \quad  \frac{j\Delta_{\epsilon,j}^2}{2\Sigma} + \frac{j\ln(\frac{2}{\delta})}{\Sigma}   
    &&\geq\quad \ln\left(\frac{2R(\frac{K}{2}-1)}{\delta}\right) 
             - \frac{(\nu_{j^*}-\nu_{j})\Delta_{\epsilon,j}}{2\sigma_0^2} \\
    &\Leftrightarrow \quad \frac{(\nu_{j^*}-\nu_{j})\Delta_{\epsilon,j}}{2\sigma_0^2} &&\geq\quad \ln\left(\frac{2R(\frac{K}{2}-1)}{\delta}\right) - \frac{j}{2\Sigma}\left(\Delta_{\epsilon,j}^2 + 2\ln\left(\frac{2}{\delta}\right)\right)\\
    &\Leftrightarrow \quad (\nu_{j^*}-\nu_{j}) &&\geq\quad \frac{\sigma_0^2}{\Sigma\Delta_{\epsilon,j}} \left[2\Sigma\ln\left(\frac{2R(\frac{K}{2}-1)}{\delta}\right) - j\left(\Delta_{\epsilon,j}^2 + 2\ln\left(\frac{2}{\delta}\right)\right)\right]
\end{align*}

Taking a union bound over all $RK$ arm-round pairs, the GP approximation 
error under Assumption~\ref{ass:gp} contributes an additive $2RK\delta_{\mathrm{GP}}$ 
to the total misidentification probability. Setting $\delta_{\mathrm{GP}} = \delta/(4RK)$ 
and replacing $\delta$ with $\delta/2$ throughout absorbs this term, yielding 
the stated bound with confidence $1 - \delta$.

\newpage
\subsection{Example}\label{appendix::example}
Consider a problem instance where $\ell_{i,t}$ is the sum of $t$ independent random variables for any $t \in \mathbb{N}$ with common mean $\mu_i$. In addition, assume each summand of $\ell_{i,t}$ is bounded and without loss of generality in $[0,1]$. Then, according to the Hoeffding inequality, we have
\begin{align*}
    \mathbb{P}\left(\left|\ell_{i,t} - t\mu_i \right| \geq c\right) \leq 2\exp\left(-\frac{2c^2}{t}\right).
\end{align*}
Dividing by $t$, with probability at least $1-\delta$, we have
\begin{align*}
    \left|\frac{\ell_{i,t}}{t} - \mu_i\right| \leq \gamma(t) := \sqrt{\frac{1}{2t}\ln\left(\frac{2}{\delta}\right)}.
\end{align*}
Computing the inverse gives us $\gamma^{-1}(y) = \frac{\ln(2/\delta)}{2y^2}$, which results in the following sample complexity bound for SH
\begin{equation}
\begin{split}\label{equation:hyperband bound example}
    N_{SH} = \,  2R 
      \max_{{j\in\mathcal{K}\setminus j^*}} j\,\left( 1+ \frac{\ln(\frac{2}{\delta})}{2\max \left\{ \frac{\epsilon}{2}, \frac{\mu_{j^*} - \mu_j}{2} \right\}^2}  \right)
\end{split}
\end{equation}
to identify the best arm with probability at least $1-\delta$.
%
%

Comparing the standard SH budget $N_{SH}$ with the bound from Theorem~\ref{theorem worst case}, denoted $N_{prior}$, highlights the budget savings offered by PSH (Algorithm \ref{alg:mf_prior_sh}) under informative priors. 
We characterize these savings depending on prior quality by analyzing when $N_{prior} < N_{SH}$ holds by solving:

\begin{equation}
\begin{split}
    \max_{j \in \mathcal{K}\setminus\{j^*\}} \frac{4R\Sigma}{\max\{\epsilon, \mu_{j^*}-\mu_j\}^2} 
    \quad & \Bigg[ \ln\left(\frac{2R(\frac{K}{2}-1)}{\delta}\right) - \frac{(\nu_{j^*}-\nu_j)\max\{\epsilon, \mu_{j^*}-\mu_j\}}{2\sigma_0^2} \Bigg] \\
    &\overset{!}{\leq} 2R\max_{{j\in\mathcal{K}\setminus j^*}} j\,\left( 1+ \frac{\ln(\frac{2}{\delta})}{2\max \left\{ \frac{\epsilon}{2}, \frac{\mu_{j^*} - \mu_j}{2} \right\}^2} \right) = N_{SH}
\end{split}
\end{equation}

\subsection{Prior Efficiency Condition}\label{appendix:prior efficiency}
\begin{corollary}[Prior Efficiency Condition]
\label{cor:prior_efficiency}
Let $N_{SH}$ be the standard Successive Halving budget and $N_{prior}$ be the sufficient budget derived in Theorem \ref{theorem worst case}. A sufficient condition for the prior-guided approach to yield budget savings (i.e., $N_{prior} < N_{SH}$) is that the prior bias satisfies:
\begin{equation}
    \nu_{j^*} - \nu_j \geq \frac{\sigma_0^2}{\Sigma \Delta_{\epsilon,j}} \left[ C_{\text{alg}}(\delta) - C_{\text{data}}(j, \delta) \right],
\end{equation}
where the algorithmic overhead $C_{\text{alg}}$ and instance-dependent complexity $C_{\text{data}}$ are defined as:
\begin{align*}
    C_{\text{alg}}(\delta) = 2\Sigma\ln\left(\frac{2R(\frac{K}{2}-1)}{\delta}\right), \quad
    C_{\text{data}}(j, \delta) = j\left(\Delta_{\epsilon,j}^2 + 2\ln\left(\frac{2}{\delta}\right)\right).
\end{align*}
\end{corollary}
Appendix \ref{appendix:comparison of sh bounds} includes the derivation of Corollary \ref{cor:prior_efficiency}.

\paragraph{Interpretation} Corollary \ref{cor:prior_efficiency} decouples the influencing factors of the budget savings. $C_{\text{alg}}$ represents the fixed overhead of the Bayesian inference mechanism, while $C_{\text{data}}(j)$ captures the hardness of distinguishing arm $j$. The inequality explicitly demonstrates that the required prior signal strength $(\nu_{j^*} - \nu_j)$ scales linearly with the prior uncertainty $\sigma_0^2$ and inversely with the problem gap $\Delta_{\epsilon,j}$. Thus, for easier problems, where $C_{\text{data}}$ is large or priors are sharper, i.e., $\sigma_0^2$ is small, the threshold for savings is accordingly lower.

We analyze Corollary~\ref{cor:prior_efficiency} to examine how the required prior 
gap $(\nu_{j^*}-\nu_j)$ shifts across precision ($\epsilon$), reliability ($\delta$), 
and belief strength ($\sigma_0$). Results are provided in 
Appendix~\ref{sec:appendix min prior gap plot}.

\newpage

\section{Details on the Experimental Setup and Prior Generation}\label{sec:appendix experimental setup}
We executed our experiments on a high performance computing cluster, using two AMD EPYC Milan 7763 CPUs with four GB memory each and 	
2.45 GHz base frequency.
\subsection{Benchmarks}\label{subsec:appendix benchmark}
\textbf{Synthetic Benchmark}
We define a synthetic search space where each arm $j$ is characterized by a logistic-like saturating learning curve approaching its true final mean $f_j(B)$. The underlying function is modeled as
\begin{equation}
    f(t) = \mu_j(1-\exp(-t/\tau)),
\end{equation}
where $\mu_j\in[0,1]$ is sampled uniformly and the time constant $\tau$ varies linearly with the arm index to simulate diverse learning dynamics. The maximum fidelity $B$ is set equal to $K=256$.
As discussed in Section \ref{sec:prior guided sh}, for the kernel-based projection of performance estimates on the maximum fidelity $B$ via learning curves, we use a saturating kernel
\begin{equation*}
    k_\mathrm{satexp} = \sigma^2f(t)f(t').
\end{equation*}
This non-stationary kernel models learning curves that exhibit diminishing returns with monotonic growth and a saturation scale. For smoothness, we add a Radial Basis Function (RBF) kernel~\citep{rasmussen-book06a}. The sum of the two kernels yields a non-stationary but flexible GP prior that extrapolates learning curves accurately by combining global saturation behavior with local smooth deviations. We evaluate the synthetic benchmark for 20 seeds.

\begin{figure}[H]
    \centering
    \includegraphics[width=0.6\linewidth]{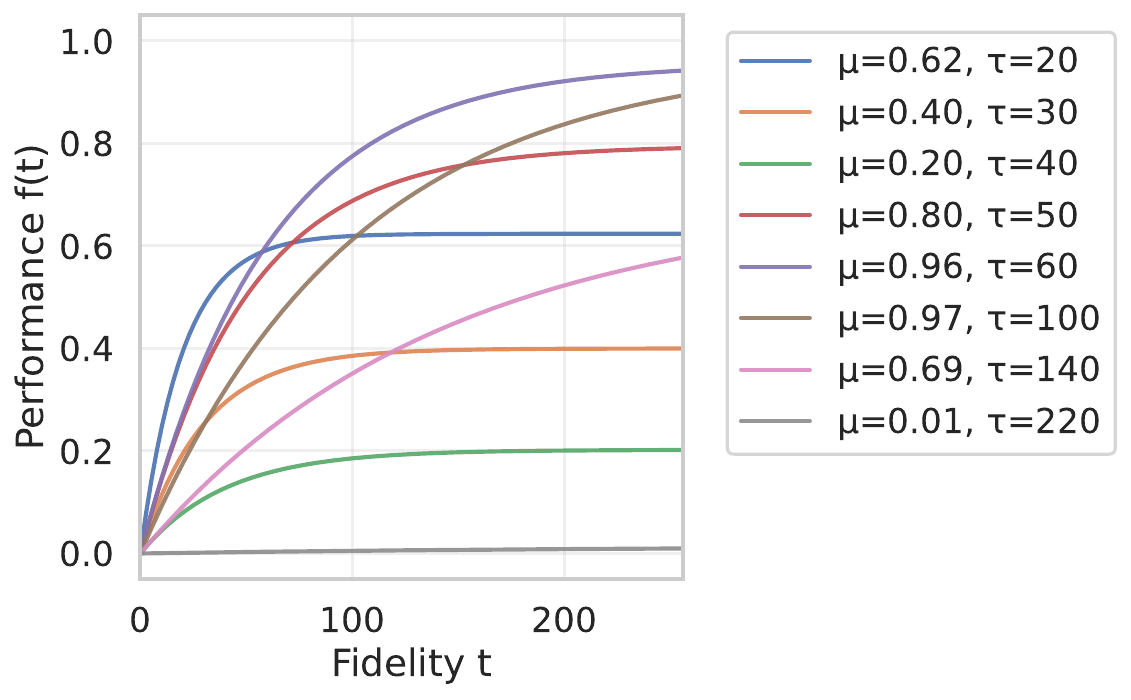}
    \caption{Learning curves of arms in the synthetic benchmark. Each arm~$j$ 
follows a smooth saturating function 
$f_j(t)$, 
where $\mu_j \in [0,1]$ is the arm's true final performance randomly sampled and the time constant 
increases linearly with the arm index ($\tau_j = 20 + 10 \cdot j$). }
    \label{fig:learning curves}
\end{figure}

\textbf{LCBench}
We use the LCBench suite~\citep{zimmer-tpami21a} from YAHPO Gym~\citep{pfisterer-automl22a}, a learning curve benchmark that contains tabular performance data for neural network configurations. We evaluate 20 seeds for each of the 34 instances. We set the maximum fidelity to $B=52$ and the number of arms to $K=256$.
For LCBench, we use a linear kernel for the learning curve modeling. Empirically, this is a good fit for the GP prediction performance-wise compared to more dynamic choices of kernels.
\subsection{Prior Generation}\label{sec:appendix prior generation}
To characterize the ways external knowledge about configuration quality can be integrated into our optimization process, we consider distinct ways of prior generation. These priors vary in terms of informativeness, sorted from highly accurate and useful to intentionally misleading.

Notably, the algorithm does not terminate within a fidelity, since the sufficient budget in Theorem \ref{theorem worst case} involves maximizing over all arms in a specific fidelity. 

\textbf{Rank-based Prior}
This strategy maps each arm $j\in\mathcal{K}$ to its inverse rank based on true final performance:
\begin{equation}
    \nu_j = 1/(\mathrm{rank}_j+1),
\end{equation}
leveraging ordinal information without requiring exact performance values.

\textbf{Performance-based Prior}
To simulate noisy expert estimates, the prior samples means from a normal distribution centered on the actual performance, representing a scenario where the user has a good but imperfect understanding of the final configuration performance:
\begin{equation}
    \nu_j\sim\mathcal{N}(f_j(B),\sigma_0^2).
\end{equation}

\textbf{Indicator Prior}
This binary strategy assigns a value of 1 to arms within an $\ epsilon$-neighborhood of the global optimum, and 0 otherwise. It tests efficiency gains provided that an expert can correctly identify regions of interest:
\begin{equation}
    \nu_j = \mathbbm{1}[f^*-f_j(B)\leq\epsilon].
\end{equation}

\textbf{Uniform Prior}
Serving as a non-informative baseline, this prior assigns the mean of all true final performance values to each arm.
\begin{equation}
\nu_{j} = \frac{1}{|\mathcal{K}|} \sum_{i \in \mathcal{K}} f_i(B).
\end{equation}
\textbf{Inverse Rank Prior}
To empirically validate the robustness of PSH (Section \ref{subsec:robustnes}), we define a misleading prior that assigns the highest means to the poorest performing arms:
\begin{equation}
    \nu_j = (\mathrm{rank}_j+1)/\mid\mathcal{K}\mid.
\end{equation}
Thus, PSH has to recover from instantiation with misleading information.

\subsection{Robustness of Prior-Guided Successive Halving}\label{subsec:robustnes_appendix}
While informative priors are utilized to accelerate our approach, practical use cases may involve uncertain or misleading user beliefs. We therefore examine the behavior of Prior-Guided Successive Halving (PSH) under different prior qualities to assess its robustness in varying settings.

\paragraph{Uninformative Priors}
We consider the uninformative setting, represented by a uniform prior with mean $\nu_j=\nu \,\,\forall j\in[K]$. In this scenario, $\nu_{j^*}-\nu_j = 0\,\,\forall j\in[K]$ holds. Consequently, the prior aligning term $\frac{(\nu_{j^*}-\nu_j)\max\{\epsilon, \mu_{j^*}-\mu_j\}}{2\sigma_0^2}$ vanishes entirely from the complexity bound. The result reduces to a term independent of the prior, recovering the behavior of standard SH. Thus, PSH defaults to the frequentist strategy when the prior provides no discerning information.

\paragraph{Misleading Priors}
Having shown how the prior facilitates efficiency gains in the favorable setting where priors are informative, we now address the misleading setting, specifically when the prior suggests a suboptimal arm $j$, i.e., $\nu_j > \nu_{j^*}$. Theorem \ref{theorem worst case} quantifies the cost of such errors, defining the sufficient budget $N$ for $\epsilon$-Best Arm Identification (BAI) with probability $1-\delta$ by decoupling the complexity into a standard logarithmic term and a prior aligning term.
In the misleading setting, the prior gap term $(\nu_{j^*} - \nu_j)$ is negative, turning the prior alignment term into a positive penalty for the budget. This scales linearly with the magnitude of error. This ensures the algorithm does not fail, but merely requires a proportional increase in budget.

The bound uses the prior variance $\sigma^2_0$ to prevent unbounded costs. If the magnitude of the prior error $(\nu_{j^*} - \nu_j)$ is potentially large, the user can increase $\sigma^2_0$ to incorporate uncertainty. In the limit, the penalty vanishes, allowing the user and algorithm to recover frequentist performance.

This way of incorporating Theorem \ref{theorem worst case} into PSH (Algorithm \ref{alg:mf_prior_sh}) as an early-stopping mechanism prevents the sampling budget from increasing beyond the pre-defined standard SH budget. PSH (Algorithm \ref{alg:mf_prior_sh}) cannot use more budget than $N$ in the misleading setting, but offers considerable speedup in the informative setting, as shown in Section \ref{sec:emipirical_results}.

\subsection{GP Approximation Across Fidelities}
We emphasize that the GP serves as a per-arm extrapolation tool. For each arm $j$, a GP is fit independently to the observed fidelity-performance pairs $\mathcal{D}_j = \{(t, y_{j,t})\}_{t=1}^{N_{\text{used}_j}}$, with fidelity as the input, and queried at $t = B$ to obtain the posterior mean $\hat{\mu}_{j,r}$ and variance $\sigma_{j,r}^2$. No information is shared across arms. This makes the approach lightweight and straightforward to integrate into SH and Hyperband.

\newpage
\section{Multi-Fidelity Prior-Guided Successive Halving (PSH) Algorithm}
\label{sec:appendix algorithm}
We provide the pseudo-code for the implementation of the Multi-Fidelity Prior-Guided Successive Halving (PSH) algorithm. PSH extends the standard Successive Halving framework by incorporating a dynamic, prior-dependent stopping criterion. In each round, the algorithm uses a Gaussian Process with an Inverse Power-Law (IPL) kernel to project observations to the target fidelity $B$ and identifies a lead candidate, also called incumbent,  $j^*$. It then evaluates a stopping threshold $N_\mathrm{stop}$, derived from the current predicted mean gaps and the injected prior belief. Define
\begin{equation}
    \tilde{\Delta}_{j,r} \leftarrow \max\{\epsilon, \hat{\mu}_{j^*,r} - 
\hat{\mu}_{j,r}\} \quad\forall j \in S_r \setminus \{j^*\}.
\end{equation}
This allows the algorithm to terminate early if the evidence sufficiently distinguishes the top configuration, while otherwise defaulting to robust pruning according to the worst case analysis in Theorem \ref{theorem worst case}:
\begin{equation*}
    N_\mathrm{stop} \geq \max_{j \in S_r \setminus \{j^*\}} 
\frac{4R\Sigma}{\tilde{\Delta}_{j,r}^2} 
\cdot \Bigg[ \ln\left(\frac{2R(\frac{K}{2}-1)}{\delta}\right) 
- \frac{(\nu_{j^*}-\nu_j)\tilde{\Delta}_{j,r}}{2\sigma_0^2} \Bigg].
\end{equation*}

\begin{algorithm}[ht]
\caption{Multi-Fidelity Prior-Guided Successive Halving (PSH)}
\label{alg:mf_prior_sh}
\begin{algorithmic}[1]
\REQUIRE Arms $\mathcal{K}$, total budget $N$,
         priors $(\nu_j, \sigma_0^2)_{j \in \mathcal{K}}$, target fidelity $B$,
         learning curve kernel $k$, elimination rate $\eta$, confidence $\delta$
\STATE $S_0 \leftarrow \mathcal{K}$, \quad $R \leftarrow \lceil \log_\eta(|\mathcal{K}|) \rceil$
\STATE $\hat{\mu}_j \leftarrow \nu_j \;\; \forall j \in \mathcal{K}$, \quad $N_{\text{used}} \leftarrow 0$, \quad $n_{\text{prev}} \leftarrow 0$
\FOR{$r = 0,\dots,R-1$}
    \STATE $n_r \leftarrow \left\lfloor \dfrac{N}{R \cdot |S_r|} \right\rfloor$
    \STATE $N_{\text{used}} \leftarrow N_{\text{used}} + |S_r| \cdot (n_r - n_{\text{prev}})$, \quad $n_{\text{prev}} \leftarrow n_r$
    \FOR{$j \in S_r$}
        \STATE Observe $y_{j,t}$ for $t = n_{\text{prev}}+1, \dots, n_r$; update $\mathcal{D}_j$
        \STATE $\hat{\mu}_{j,r} \leftarrow \mathbb{E}[f_j(B) \mid \mathcal{D}_j]$, \quad
               $\sigma_{j,r}^2 \leftarrow \mathrm{Var}[f_j(B) \mid \mathcal{D}_j]$
    \ENDFOR
    \STATE $j^* \leftarrow \argmax_{j \in S_r} \hat{\mu}_{j,r}$
    \STATE Compute $N_\text{stop}$
    \IF{$N_{\text{used}} \geq N_{\mathrm{stop}}$}
        \STATE \textbf{return} $j^*$
    \ENDIF
    \STATE $S_{r+1} \leftarrow$ top $\left\lceil |S_r| / \eta \right\rceil$ arms from $S_r$ by $\hat{\mu}_{j,r}$
\ENDFOR
\STATE \textbf{return} $j^*$
\end{algorithmic}
\end{algorithm}

\subsection{Hyperband Extension} \label{sec:appendix:hyperband}
To empirically show that our stopping criterion is compatible with Hyperband \citep{li-jmlr18a} and PriorBand, we conduct an empirical evaluation as discussed in \cref{sec:emipirical_results} on LCBench \citep{zimmer-tpami21a}. Importantly, this requires a slight adjustment of the stopping criterion with $\log_\eta(K)$ instead of $\log_2(K)$.
Importantly, the termination of a successive-halving bracket does not impact the following brackets. Upon termination, we also guarantee that the configuration is evaluated on the maximum fidelity, to allow returning the incumbent across all brackets. Note that while we expect an overall prior map to be provided to Hyperband, priors are then created with respect to their own bracket. The pseudocode is provided below, with $CREATE\_HYPERBAND\_BRACKETS$ construction brackets as discussed by \citet{li-jmlr18a}.

Evaluation results can be found in Figure~\ref{fig:hyperband prob}. The PSH stopping criterion was successfully integrated into Hyperband as described in Algorithm~\ref{alg:hyperband_psh}, allowing individual brackets to terminate early once sufficient budget has been consumed relative to $N_\text{stop}$.

Figure~\ref{fig:hyperband prob} shows regret and probability of $\epsilon$-optimal arm identification across prior types, where $\epsilon$-BAI denotes the fraction of runs in which the returned arm falls within an $\epsilon$-ball of the optimal arm. Note that the probability is bounded by the hardness of the problem, not our method.

Rank, performance, and indicator priors achieve comparable regret of ${\approx}0.06$ at substantially lower consumed budget (${\approx}650$--$850$) than inverse rank and uniform (${\approx}1000$--$1100$), confirming that informative priors enable earlier stopping without sacrificing solution quality. Equal to the SH results in Section~\ref{subsec:results}, the rank prior results in the highest budget savings with some distance. The $\epsilon$-BAI results are largely consistent across methods, with the exception of inverse rank which drops to ${\approx}0.12$.

The results demonstrate that PSH's stopping criterion transfers successfully to the Hyperband setting, preserving the same pattern observed in standalone SH of offering substantial budget savings without sacrificing on performance.

\begin{figure}[H]
    \centering
    \begin{subfigure}[t]{0.49\linewidth}
        \centering
        \texttt{Hyperband Regret}
        \includegraphics[width=\linewidth,trim={0cm 0.25cm 0cm 0cm},clip]{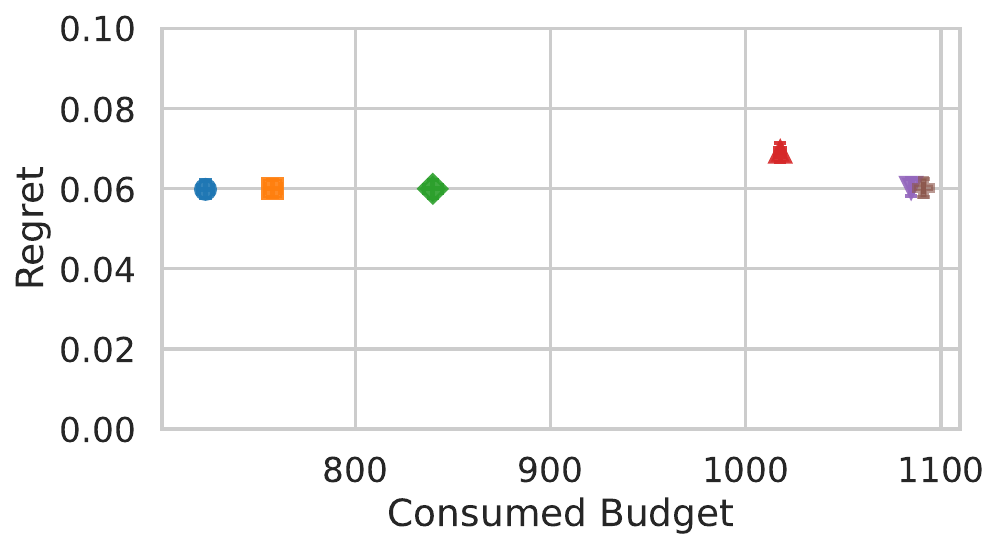}
    \end{subfigure}
    \hfill
    \begin{subfigure}[t]{0.49\linewidth}
        \centering
        \texttt{Hyperband Probability of $\epsilon$-BAI}
        \includegraphics[width=\linewidth,trim={0cm 0.25cm 0cm 0cm},clip]{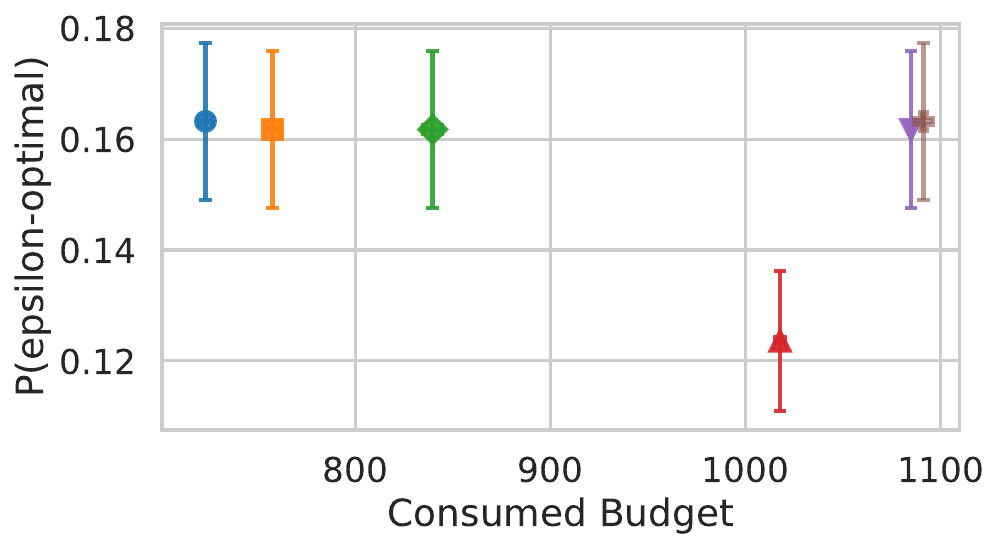}
        \label{fig:lcbench}
    \end{subfigure}
    \includegraphics[width=\linewidth,trim={0cm 0cm 0cm 0.25cm},clip]{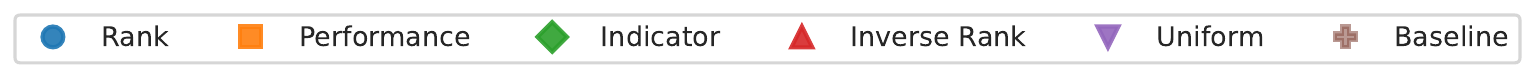}
    \caption{Hyperband performance across prior types on regret (left) and probability of $\epsilon$-optimal arm identification (right), plotted against consumed budget.}
    \label{fig:hyperband prob}
\end{figure}

\begin{algorithm}[ht]
\caption{Hyperband with PSH (Simplified)}
\label{alg:hyperband_psh}
\begin{algorithmic}[1]
\REQUIRE Maximum resource $B$, elimination ratio $\eta$, 
         sampler $\mathcal{P}$, prior map $\pi$,
         learning curve kernel $k$, confidence $\delta$
\STATE $\mathcal{B} \leftarrow \textsc{create\_hyperband\_brackets}
       (B, \eta, \mathcal{P}, \pi)$
\STATE $\mathcal{W} \leftarrow \emptyset$
\FOR{$\bigl(\mathcal{K}_s, N_s, (\nu_j,\sigma_0^2)_j\bigr) \in \mathcal{B}$}
    \STATE $j^*_s \leftarrow \textsc{PSH}\bigl(\mathcal{K}_s,\, N_s,\,
           (\nu_j,\sigma_0^2)_{j \in \mathcal{K}_s},\, B,\, k,\, 
           \eta,\, \delta\bigr)$
           \hfill\COMMENT{Algorithm~\ref{alg:mf_prior_sh}}
    \STATE $\mathcal{W} \leftarrow \mathcal{W} \cup \{j^*_s\}$
\ENDFOR
\STATE \textbf{return} $\arg\max_{j \in \mathcal{W}} \hat{\mu}_{j}(B)$
\end{algorithmic}
\end{algorithm}

\subsection{PriorBand Extension}\label{sec:appendix:priorband}
We also demonstrate compatibility with PriorBand~\citep{mallik-neurips23a}, which combines Hyperband with prior-informed configuration sampling and sequential bracket execution. One key adaptation is required: since PSH's early stopping criterion can substantially reduce the number of configurations evaluated per bracket, we cannot rely on PriorBand's original condition for activating incumbent-based sampling, which requires one configuration to have been evaluated at the highest fidelity. Instead, we activate incumbent-based sampling after the first bracket completes. The same prior type used for PriorBand's configuration sampling is used to assign GP prior means within each PSH bracket, ensuring consistency between the two sources of prior information. Otherwise, we follow the original PriorBand procedure.

To construct the PriorBand prior, we sample $1000$ configurations uniformly, evaluate each at maximum fidelity, and select the best-performing configuration as the prior centre. A Normal distribution is then placed over each hyperparameter centred on this configuration with standard deviation $\sigma = \frac{u - l}{5}$, where $u$ and $l$ are the hyperparameter's upper and lower bounds respectively, yielding a broad but informed prior over the configuration space.

Figure~\ref{fig:priorband} shows regret for PriorBand with the PSH early stopping criterion. The rank prior triggers the earliest stopping (${\approx}830$) but at the cost of slightly higher regret (${\approx}0.03$), while inverse rank and indicator achieve comparable regret of ${\approx}0.02$ at moderate budget (${\approx}975$--$1000$). Performance consumes more budget (${\approx}1050$) and achieves the lowest regret (${\approx}0.017$), with the uninformative uniform prior and no-prior baseline behaving similarly at 1100 iterations and regret of ${\approx}0.015$.

We emphasize that the integration with PriorBand is presented as a demonstration of compatibility rather than a claim of joint optimality. PriorBand's configuration sampling and PSH's arm identification operate on fundamentally different objectives: PriorBand shapes \emph{which} configurations enter the search, while PSH determines \emph{when} sufficient evidence has accumulated to stop evaluating them. Because these objectives are not jointly optimized, a precise theoretical analysis of their interaction is out of scope. The relevant observation is that PSH's stopping criterion transfers gracefully to the PriorBand setting, consistently reducing the consumed budget while the regret-budget tradeoff remains Pareto-competitive across all prior types.

\begin{figure}[H]
    \centering
    \texttt{Priorband with early stopping criterion}
    \includegraphics[width=0.7\linewidth]{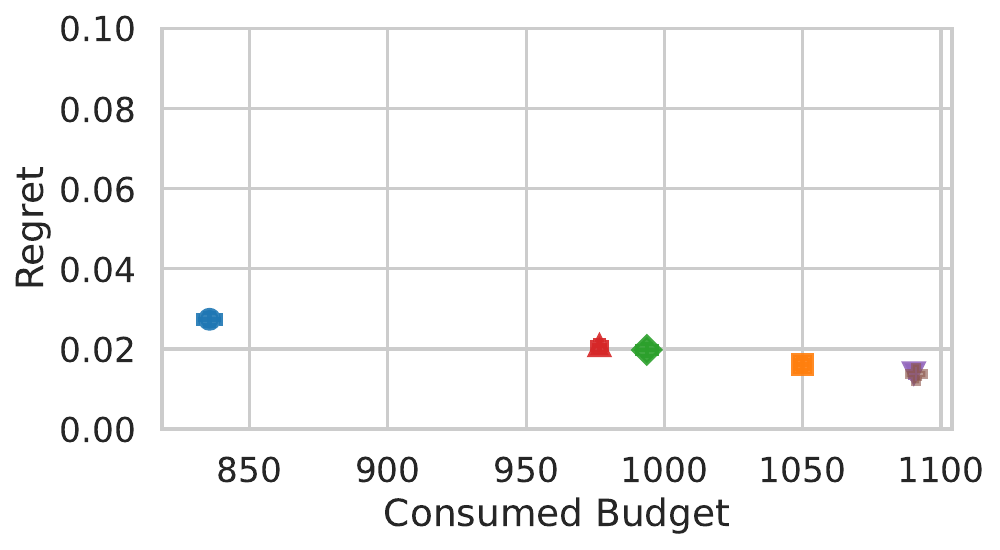}
    \includegraphics[width = \linewidth]{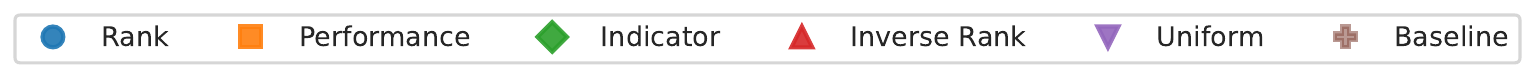}
    \caption{Regret for PriorBand with the PSH early stopping criterion across prior types, plotted against consumed budget.}
    \label{fig:priorband}
\end{figure}

\newpage
\section{Analysis of the Minimum Required Prior Gap}
\label{sec:appendix min prior gap plot}

This section provides a sensitivity analysis of the required prior gap in
\begin{equation}
    (\nu_{j^*}-\nu_{j}) \geq\frac{\sigma_0^2}{\Sigma\Delta_{\epsilon,j}} 
    \left[2\Sigma\ln\!\left(\frac{2R(\frac{K}{2}-1)}{\delta}\right) 
    - j\!\left(\Delta_{\epsilon,j}^2 + 2\ln\!\left(\frac{2}{\delta}\right)\right)\right].
\end{equation}
Figure~\ref{fig:prior gap plots} illustrates how the minimum necessary separation 
between the lead candidate $j^*$ and suboptimal arms $j$ scales with $\delta$, 
$\sigma_0$, and $\epsilon$.

\paragraph{Impact of Confidence ($\delta$)} As $\delta\rightarrow 0$, the required 
prior gap increases logarithmically. Stricter confidence guarantees impose higher 
algorithmic overhead, requiring a stronger prior signal to substitute for extensive 
sampling to rule out failure.

\paragraph{Impact of Prior Uncertainty ($\sigma_0$)} As $\sigma_0$ increases, 
representing a less confident prior, the required gap grows monotonically. The 
algorithm shifts weight toward empirical terms, demanding stronger prior separation 
to achieve the same budget saving.

\paragraph{Impact of Error-Tolerance ($\epsilon$)} As $\epsilon$ approaches zero, 
the required prior gap increases sharply, consistent with the inverse dependence on 
$1/\Delta_{\epsilon,j}$. Relaxed precision requirements or stronger priors thus 
yield the largest efficiency gains.

\begin{figure}[]
    \centering
    \begin{subfigure}[b]{0.6\columnwidth}
        \centering
        \includegraphics[width=\textwidth]{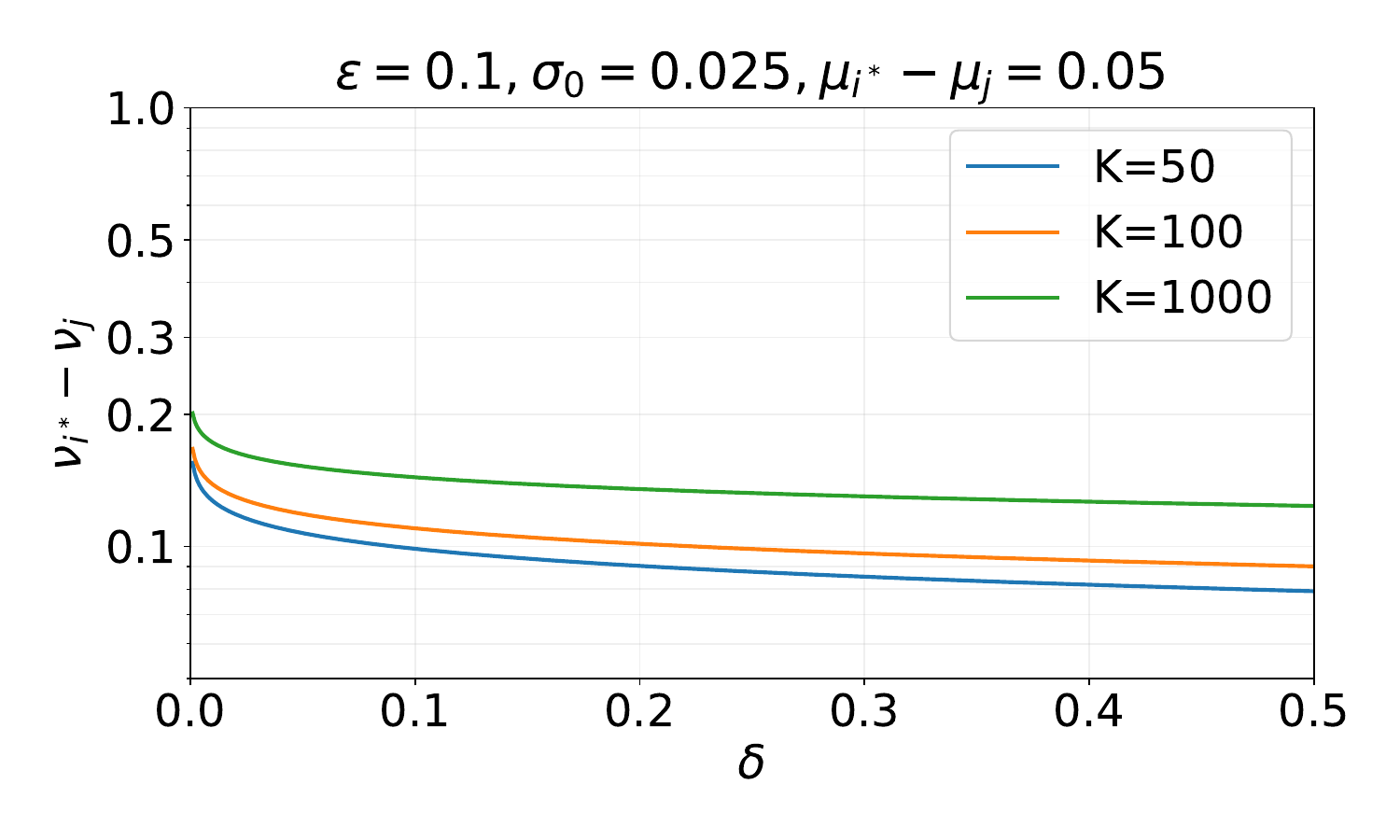}
        \caption{Impact of $\delta$}
        \label{fig:delta}
    \end{subfigure}
    \hfill
    \begin{subfigure}[b]{0.6\columnwidth}
        \centering
        \includegraphics[width=\textwidth]{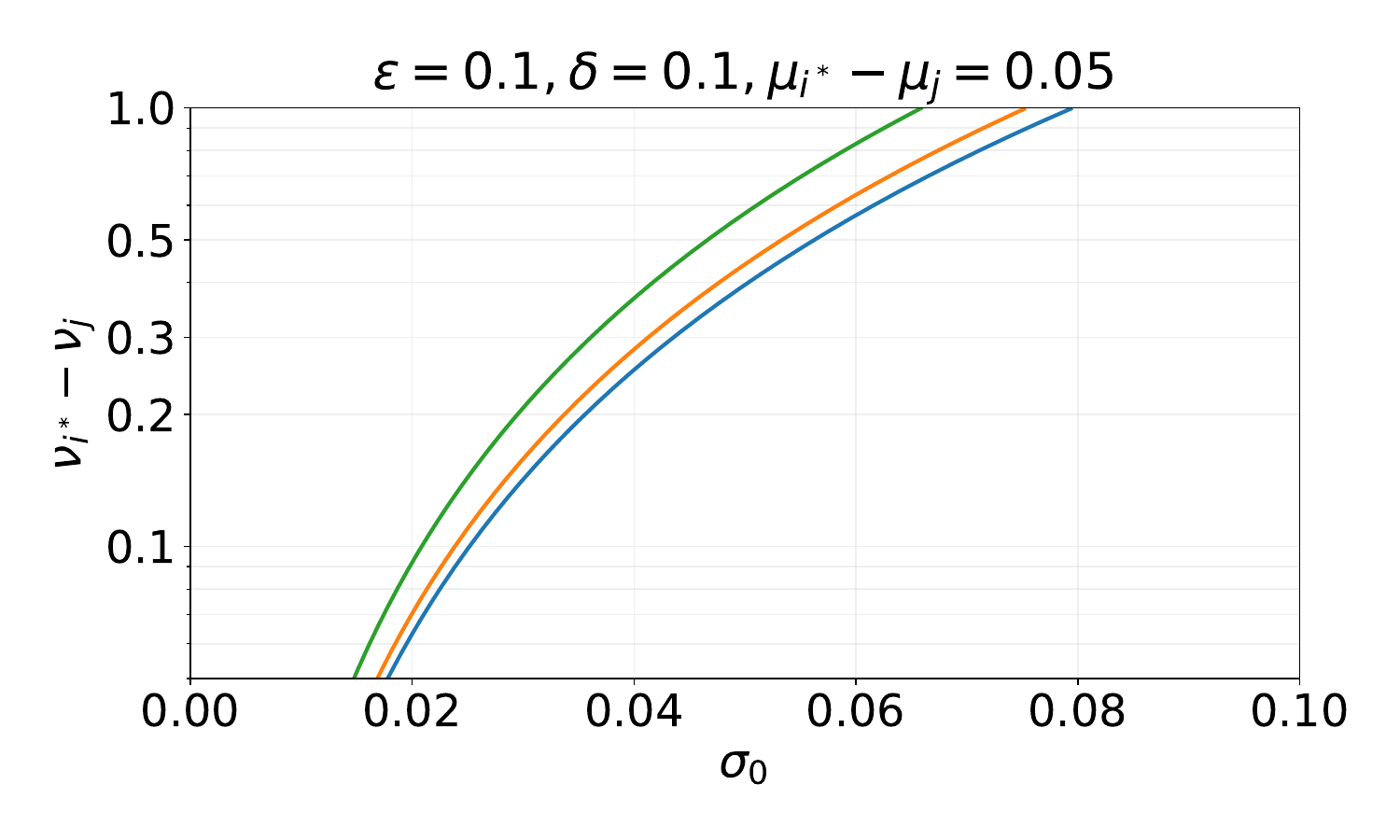}
        \caption{Impact of $\sigma_0$}
        \label{fig:sigma}
    \end{subfigure}
    \begin{subfigure}[b]{0.6\columnwidth}
        \centering
        \includegraphics[width=\textwidth]{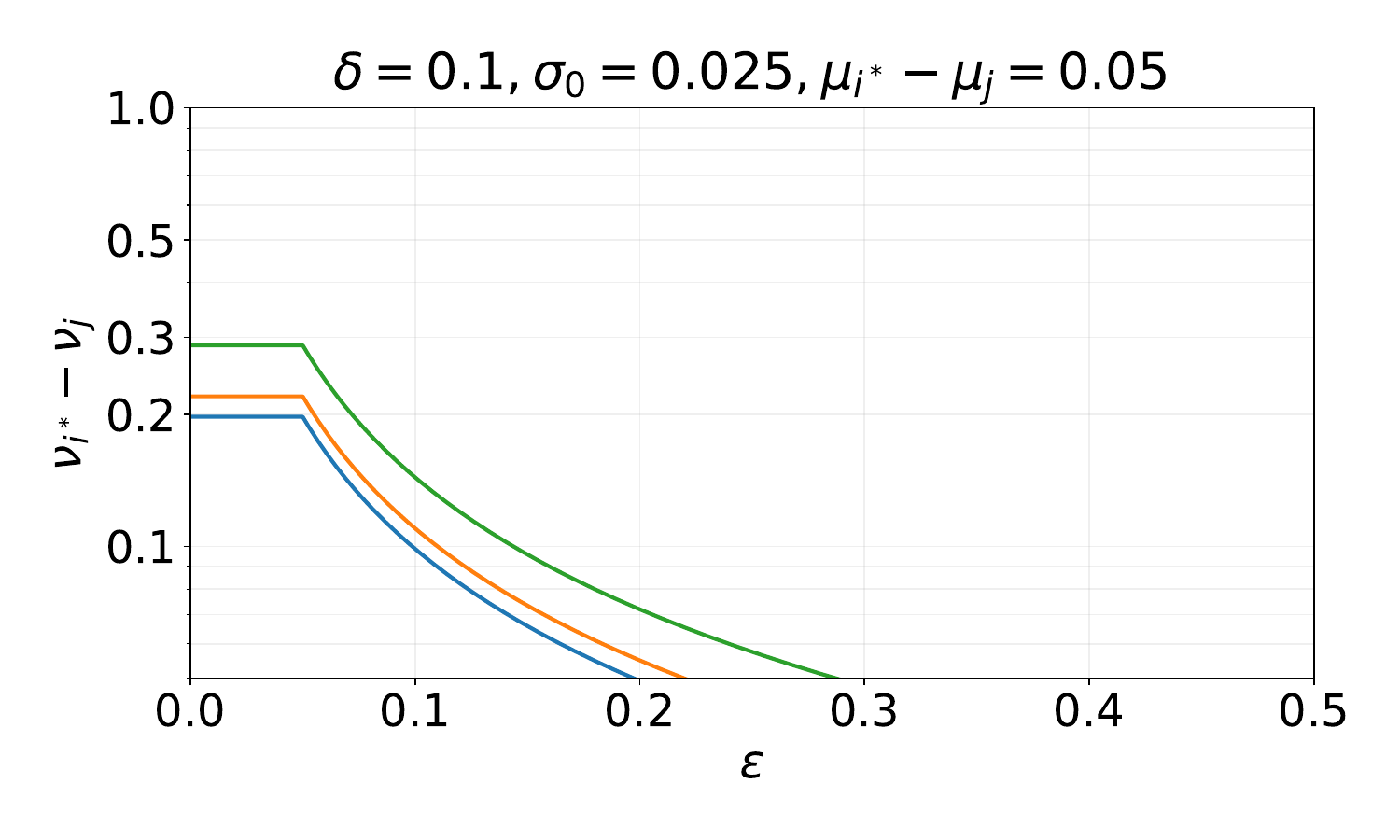}
        \caption{Impact of $\epsilon$}
        \label{fig:epsilon}
    \end{subfigure}
    \caption{Sensitivity analysis of the required prior gap across confidence 
    ($\delta$), uncertainty ($\sigma_0$), and precision ($\epsilon$).}
    \label{fig:prior gap plots}
\end{figure}

\newpage
\section{Further Results}\label{sec:appendix further results}

\subsection{LCBench Ablation with Different Kernels}\label{subsec:appendix kernel ablation}
In Figure \ref{fig:kernel ablation}, we show the LCBench regret distributions and consumed budgets for all prior generation schemes when varying the kernel of the Gaussian Process surrogate. Overall, regret remains low and stable across all kernels, with minor variations per prior scheme. Across all prior types, the linear kernel achieves the best results for rank and performance priors. If evaluated with the other prior kinds, no kernel outperforms the others.
\begin{figure}[H]
    \centering
    \includegraphics[width=\linewidth]{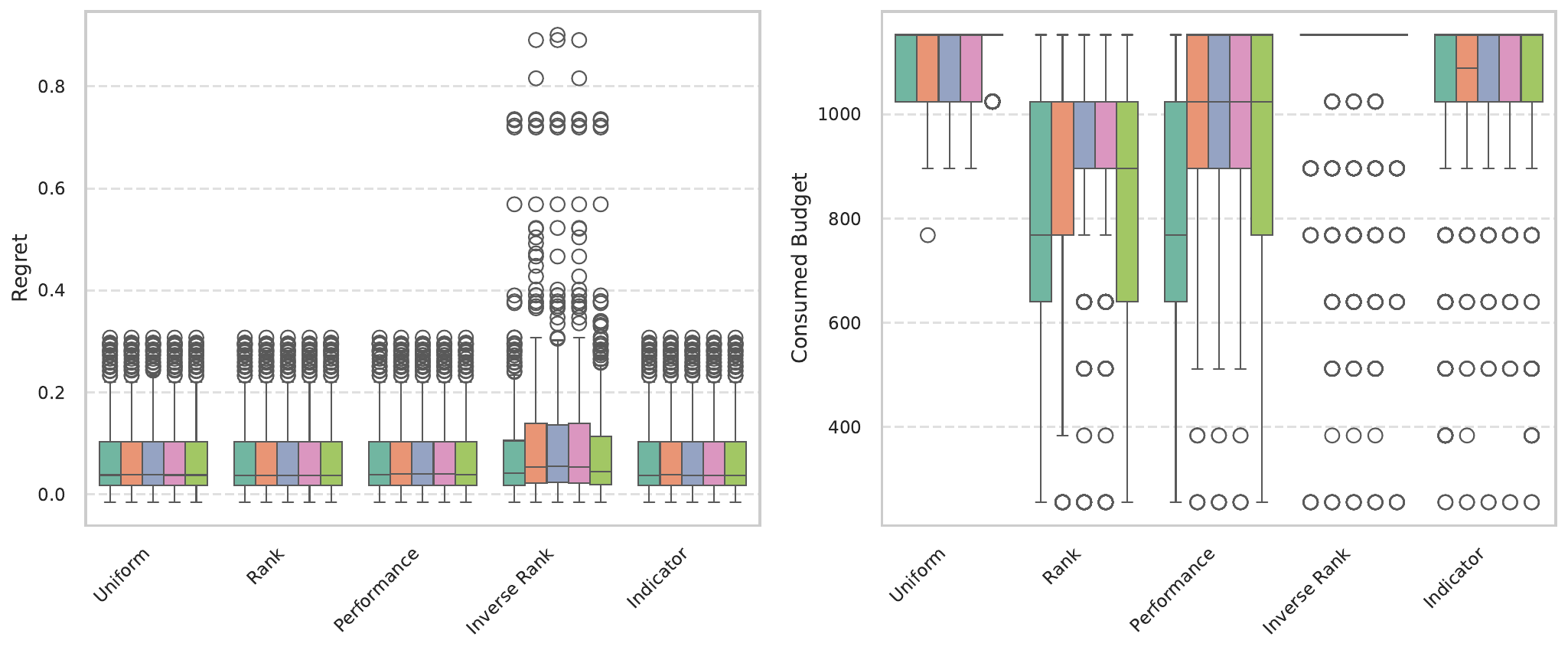}
    \includegraphics[width=0.8\linewidth]{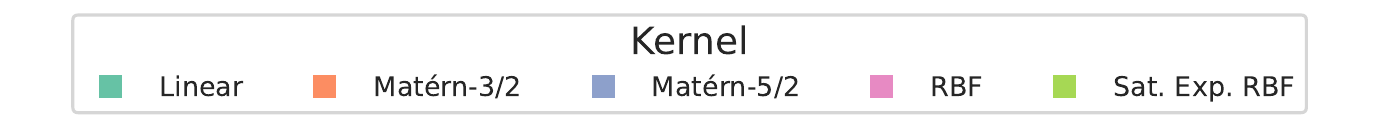}
    \caption{Regret and consumed budget of PSH on LCBench with varying GP surrogate kernels.}
    \label{fig:kernel ablation}
\end{figure}

\subsection{Ablation for Varying Sigma}\label{subsec:appendix sigma ablation}

Figure \ref{fig:sigma ablation} shows the regret distributions and consumed budgets for all five PSH prior generation schemes. On the synthetic benchmark, all prior kinds achieve near-zero regret, with lower regret for smaller $\sigma_0$. Consistent with Figure \ref{fig:empirical_resultsparte_front_lcbench}, rank priors yield a substantial reduction in budget consumption across all $\sigma_0$ values. While this trend also holds for performance priors, they are more sensitive to user uncertainty, as $\sigma_0$ visibly affects their result. Uniform, inverse-rank, and indicator priors do not appear to result in early stopping, as the full budget is executed (with outliers in the case of indicator priors).

LCBench exhibits higher, yet very stable, regret distributions across all priors, with different prior types yielding almost identical regret. Mirroring the synthetic results, rank and performance priors lead to the most consistent budget reductions, with lower $\sigma_0$ again proving beneficial. In order of decreasing effect, indicator, uniform, and even inverse rank priors lead to early termination.
Overall, these results demonstrate that PSH provides consistently strong and budget-efficient performance on both benchmarks, with $\sigma_0$ behaving as intended and only marginally affecting
stability—confirming robustness under varying user uncertainty. Lower uncertainty yields a corresponding speedup, while higher uncertainty makes PSH more cautious.
\begin{figure}[H]
    \centering
    \texttt{Synthetic Results}
    \includegraphics[width=\linewidth,trim={0cm 0cm 0cm 0cm},clip]{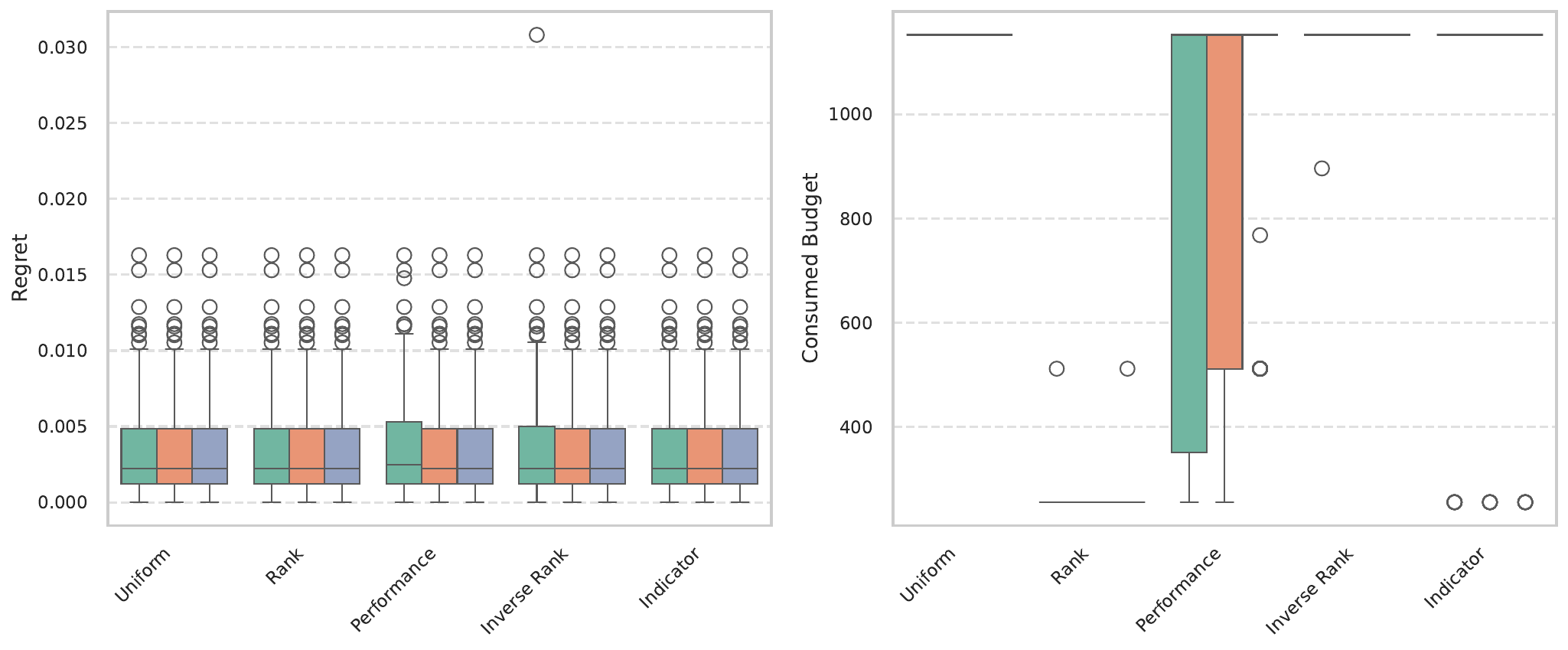}
    \texttt{LCBench Results}
    \includegraphics[width=\linewidth]{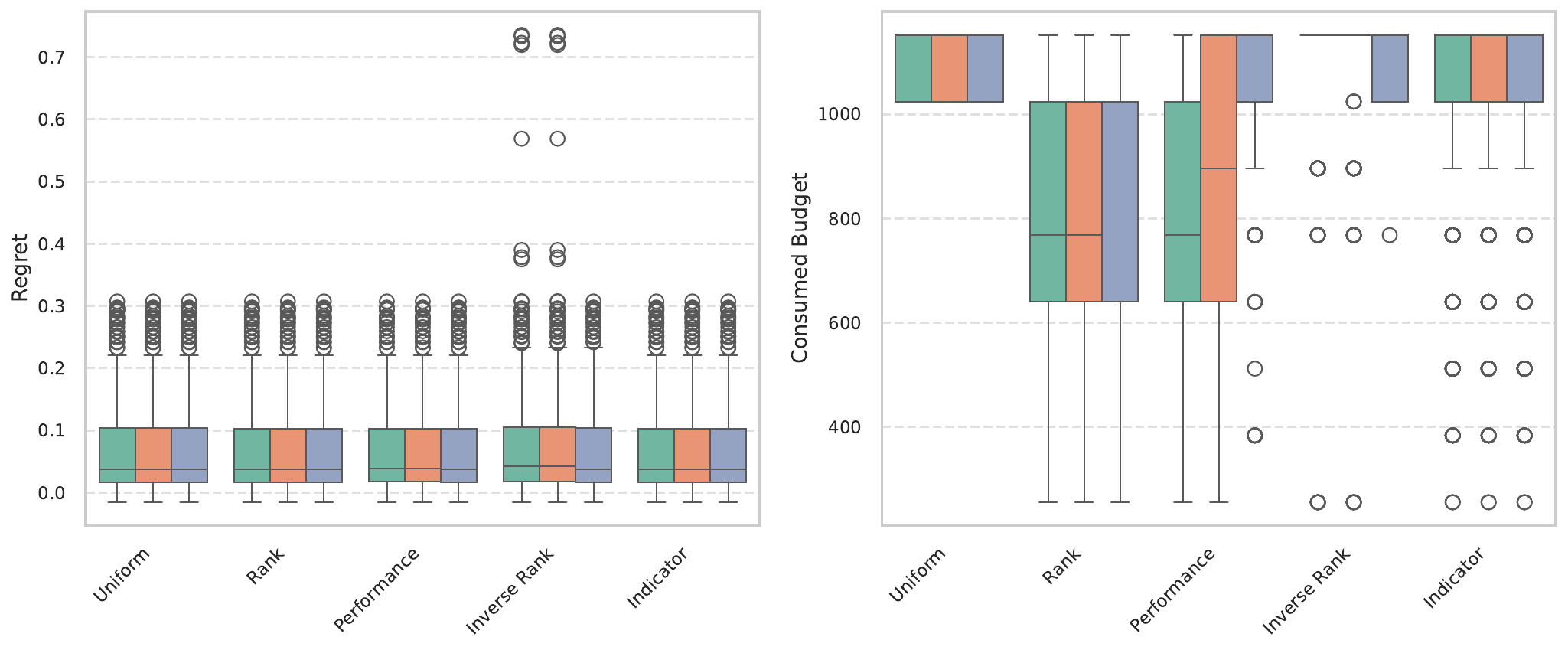}
    \includegraphics[width=0.8\linewidth]{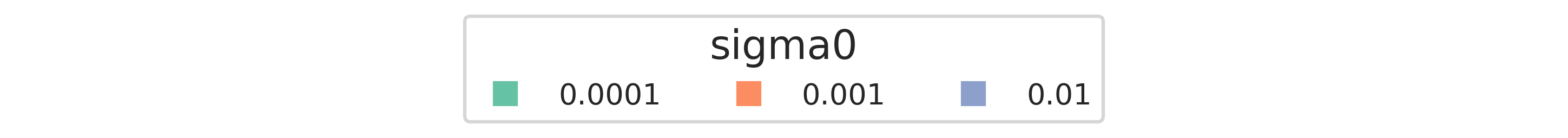}
    \caption{Regret distributions and consumed budgets for all PSH prior generation schemes under varying $\sigma_0$ on the synthetic benchmark (top) and LCBench (bottom).}
    \label{fig:sigma ablation}
\end{figure}

\newpage

\end{document}